\documentclass[letterpaper,twocolumn,10pt]{article}
\usepackage{usenix}

\usepackage{graphicx}
\usepackage{booktabs}
\usepackage{wrapfig}

\usepackage{array}

\usepackage{aliascnt}
\makeatletter
\@ifpackageloaded{complexity}{}{%
\usepackage[classfont=roman,langfont=roman,funcfont=roman]{complexity}}
\makeatother

\usepackage{mathtools}
\usepackage{bbm}
\usepackage{amsmath,amsfonts,amsmath,color,boxedminipage,url,xparse}

\usepackage{xspace}
\usepackage{nicefrac}
\usepackage{hyperref}
\usepackage{cleveref}
\usepackage{tikz}

\usepackage{subcaption}

\usepackage{tikz}
\usepackage{amsmath,amssymb,amsthm,mathtools}

\usepackage{filecontents}
\usepackage[numbers]{natbib}

\def \IsDraft{} 

\providecommand{\remove}[1]{}

\ifdefined\IsDraft
    \newcommand{\authnote}[2]{{\bf [{\color{red} #1's Note:} {\color{blue} #2}]}}
\else
    \newcommand{\authnote}[2]{}
\fi

\newenvironment{mybox}{\begin{center}\begin{tabular}{|p{0.97\linewidth}|c|}   \hline} {  \\ \hline \end{tabular} \end{center}}

\let\originalleft\left
\let\originalright\right
\renewcommand{\left}{\mathopen{}\mathclose\bgroup\originalleft}
\renewcommand{\right}{\aftergroup\egroup\originalright}

\newcommand{\condition}{\;\ifnum\currentgrouptype=16 \middle\fi|\;}

\newcommand{\eps}{\varepsilon}

\newcommand{\MathAlg}[1]{\mathsf{#1}}

\newcommand{\Ensuremath}[1]{\ensuremath{#1}\xspace}

\newcommand{\MathAlgX}[1]{\Ensuremath{\MathAlg{#1}}}

\def\FullBox{$\Box$}
\def\qed{\ifmmode\qquad\FullBox\else{\unskip\nobreak\hfil
\penalty50\hskip1em\null\nobreak\hfil\FullBox
\parfillskip=0pt\finalhyphendemerits=0\endgraf}\fi}

\def\qedsketch{\ifmmode\Box\else{\unskip\nobreak\hfil
\penalty50\hskip1em\null\nobreak\hfil$\Box$
\parfillskip=0pt\finalhyphendemerits=0\endgraf}\fi}

{\proof[#1]\list{}{\leftmargin=1cm\rightmargin=1cm}}
{\endproof%
  \vspace{10pt}
}

\renewcommand{\Pr}{\mathbb{P}}

\newcommand{\Adv}{\MathAlgX{Adv}}

\let\oldnorm\norm
\def\norm{\@ifstar{\oldnorm}{\oldnorm*}}
\makeatother

\def\norm#1{\mathopen\| #1 \mathclose\|}

\usepackage{graphicx} 

\usepackage{amsfonts}
\usepackage{bm}
\usepackage{dsfont}
\usepackage{amsmath}
\usepackage{amsthm}
\usepackage{nicefrac}
\usepackage{bbm}

\usepackage{graphicx}
\usepackage{hyperref}
\usepackage{amssymb}
\usepackage{wrapfig}
\usepackage{caption}
\usepackage{subcaption}
\usepackage[linesnumbered, ruled,vlined]{algorithm2e}

\usepackage{cleveref}

\newtheorem{example}{Example}

\newtheorem{definition}{Definition}

\newtheorem{lemma}{Lemma}
\newtheorem{corollary}{Corollary}

\newtheorem{remark}{Remark}

\usepackage{color}   
\usepackage{hyperref}
\hypersetup{colorlinks=true,linkcolor=blue}

\newcommand{\DIME}{\textbf{\textsf{DIME}}\xspace}

\begin{document}

\date{}

\title{\Large \bf \DIME: Query-Efficient Framework for Membership Inference on Diffusion Models}

\author{{\rm Tue Do}, {\rm Daniel Alabi}\\
University of Illinois at Urbana-Champaign}

\maketitle

\begin{abstract}
Membership inference attacks expose whether individual records were used to train a model, yet existing attacks on diffusion models are largely heuristic and can require substantial query budgets. We introduce \DIME (\textbf{D}enoiser \textbf{I}deal \textbf{M}embership \textbf{E}rror), a theoretically grounded and query-efficient framework for membership inference on diffusion models. Our starting point is an exact characterization of the optimal diffusion denoiser for a finite training set, which reveals that membership leakage is governed by the denoiser's implicit reconstruction error. This error decomposes into two complementary signals: a \emph{bias term}, capturing reconstruction accuracy, and a previously unexplored \emph{local crowding term}, capturing the geometry of nearby training examples. Both admit efficient estimators using only model queries, yielding a practical attack with as few as two queries. Across CIFAR-10/100, STL10-U, CelebA, and ImageNet, \DIME consistently outperforms prior attacks at comparable or substantially lower query cost, improving TPR at 1\% FPR by up to \textbf{$3\times$}; remarkably, its two-query variant can outperform existing 30-query baselines. Finally, we suggest,
discuss, and evaluate specific defenses to counteract such powerful membership tests.
\end{abstract}

\section{Introduction}
\label{sec:introduction}

Diffusion models~\citep{song2019generative,song2020score,ho2020denoising} have become one of the dominant paradigms for high-dimensional generative modeling, powering widely deployed systems for image synthesis~\citep{rombach2022high,ramesh2022hierarchical,saharia2022photorealistic}, and increasingly extending to audio, video, and scientific domains. In addition to sampling quality, diffusion models beat earlier pixel-level autoregressive approaches~\citep{vandenoord2016conditional,vandenoord2016pixel} in sampling speed and throughput.

As with most large-scale machine learning systems, the specific composition of a deployed model's training data is typically not disclosed, making it difficult for an individual, creator, or auditor to independently verify whether particular data was used in training. This opacity raises privacy concerns of several distinct kinds. 
First, whether a particular individual's data was used in training can itself be a sensitive fact independent of the specific content involved: if a model is trained on medical imaging, biometric data, or any dataset implicitly associated with a sensitive category, confirming that a specific individual's data was included can constitute a privacy harm on its own~\citep{homer2008resolving}.
Second, verifying whether a model still reflects a specific individual's data after a deletion request is a practical necessity under data-protection regulation~\citep{gdpr2016,hipaa1996,ccpa2018}, and one an external party generally cannot resolve without some way to test the model's behavior directly. Third, and increasingly salient specifically for generative image models, diffusion models have been shown not merely to learn generalizable visual concepts from their training data, but in some cases to \emph{memorize} and reproduce specific training examples~\citep{carlini2023extracting,somepalli2023diffusion}, raising direct concerns about the copyright status and provenance of generated content that have already surfaced in litigation~\citep{vincent2023ai}.

\emph{Membership inference} offers a formal, auditable way to address each of these concerns: given a trained model and a candidate data point, determine whether that point was used during training~\citep{shokri2017membership}. A successful membership inference attack against a deployed diffusion model would allow an individual, artist, or auditor to establish,  
\textit{without access to the training pipeline itself}, 
if a specific image was used, making membership inference a central tool both for attacking and for auditing the privacy properties of these models.
 
Existing membership inference attacks against diffusion
models~\citep{duan2023diffusion,kong2024efficient,rao2026scorebased,matsumoto2023membership} have made substantial empirical progress, but share a common limitation: the test statistics they employ, including loss values and single-step reconstruction error are chosen based on intuition and empirical validation, without a derived connection to what the \emph{best possible} statistic would be for this task. This leaves open a basic question: 
\begin{quote}
\emph{Given a diffusion model's specific training objective and architecture, what does an optimal membership inference statistic actually look like, and can existing empirical attacks be understood as approximations to it?}
\end{quote}

We answer this question directly. Our first contribution is theoretical, and proceeds in three steps.

First, we derive an exact, closed-form characterization of the
\emph{idealized denoiser}: the MSE-optimal noise-prediction function a diffusion model would compute at any query point, given a finite training set. We show it has an explicit, interpretable form: a responsibility-weighted average over training examples, where each example's weight is a softmax
over its distance to the query. Every prior diffusion membership inference attack evaluates the \emph{real} denoiser at a single point and treats its output as an empirically useful signal; we instead start from what the \emph{optimal} denoiser provably computes, and derive a test statistic from that closed form directly.

Second, we show this idealized denoiser admits an \emph{exact}
bias-variance decomposition of a natural estimation
error: how far a candidate point's model-implied reconstruction deviates
from the candidate itself. The bias term is exactly the response magnitude used, in various forms, by every existing attack we compare against. The variance (crowding) term, measuring how densely the candidate is surrounded by other similar training points, has no analog anywhere in the diffusion membership inference literature. Because this analytical decomposition is exact rather than heuristic, both terms are available from the very same closed-form object used to justify the first term alone, meaning existing attacks have been discarding a mathematically principled, complementary source of signal we capture.
See Figure~\ref{fig:dime-overview}.

Third, we show both terms admit efficient estimators from forward evaluations of the trained network alone, requiring no gradient access, no shadow models, and no full generative sampling. The bias term is estimated by Monte Carlo averaging, and the variance (crowding) term via a Hutchinson trace estimator applied to the denoiser's local Jacobian, using the same perturbed queries already spent on the first estimate. This construction is what makes the theory practically actionable.

Our second contribution is practical: we introduce \DIME
(\textbf{D}enoiser \textbf{I}deal \textbf{M}embership \textbf{E}rror), a test statistic derived directly from this decomposition, together with a query-efficient estimator of each component. We measure query cost as the number of forward passes through the trained model; consistent with
prior diffusion membership inference
work~\citep{rao2026scorebased,matsumoto2023membership,duan2023diffusion,kong2024efficient}, reflecting a \emph{gray-box} threat model in which the adversary has
query access to the model's noise-prediction output at chosen inputs and timesteps, but no access to weights, gradients, or training infrastructure (Section~\ref{sec:threat_model}).

A realistic adversary querying a hosted or commercially deployed diffusion model faces per-query cost, rate limits, or usage monitoring that a query-heavy
attack may not survive in practice. This also distinguishes our threat model from substantially more expensive alternatives: shadow-model
attacks require the adversary to train multiple auxiliary models~\citep{shokri2017membership}, and reconstruction-based attacks require full iterative sampling via the reverse diffusion process per candidate point~\citep{carlini2023extracting,pang2023black}; orders of magnitude more expensive than a single forward pass, and correspondingly harder to deploy unnoticed or at scale. The same efficiency matters
symmetrically for the auditing use case motivating this work: a privacy auditor verifying a deletion request or investigating potential memorization benefits from an attack that is cheap and fast to run repeatedly, not one that requires retraining models or generating samples per candidate point. We show (Section~\ref{sec:experiments}) that \DIME achieves SoTA attack performance with as few as $2$ total queries, and increasing in effectiveness as much as query budget allows.

Across five diffusion checkpoints and datasets with substantially different structure, \DIME outperforms the strongest prior baseline on every standard metric (ASR, AUC, and TPR@1\%FPR) at matched or lower query cost; we observe \DIME to improve
TPR@1\%FPR by up to $3\times$ on CIFAR-10 from SoTA methods while remaining query efficient. 
We additionally evaluate DIME against differential privacy~\citep{dwork2006differential,DworkKMMN06}, the standard formal algorithmic defense
against privacy adversaries, and find that differential privacy remains an effective defense even against our empirically strongest attack.

\begin{figure*}
    \centering
    \includegraphics[width=\textwidth]{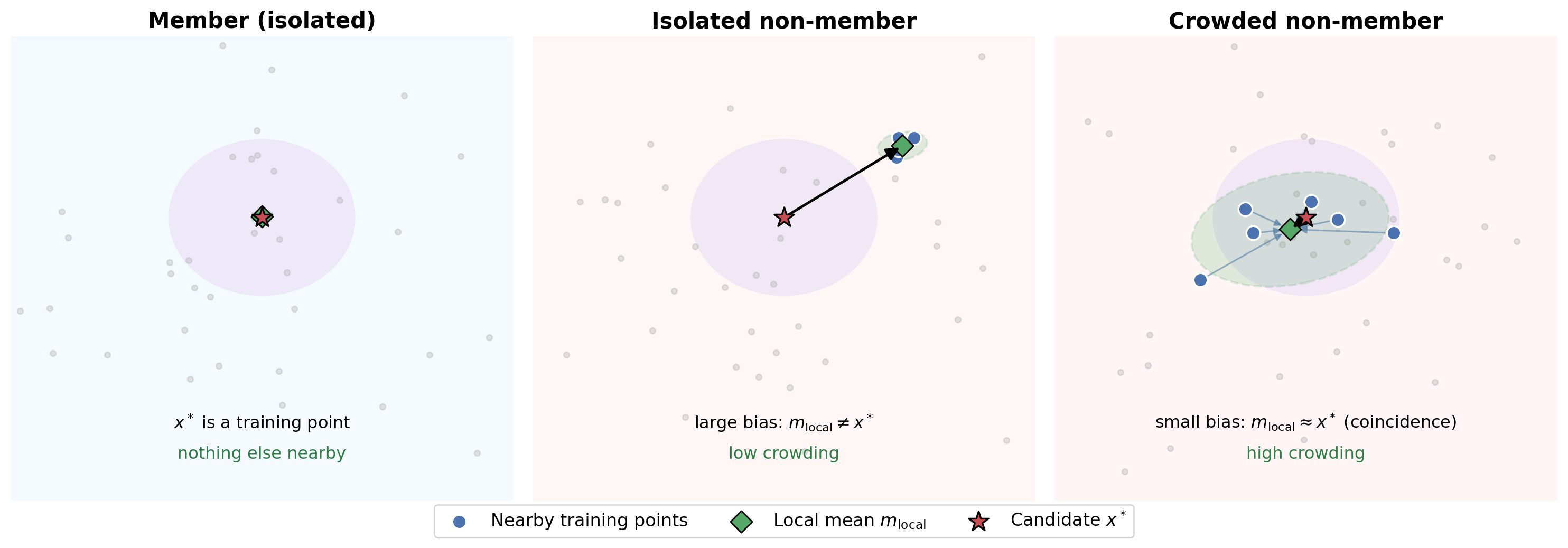}
    \caption{\textbf{Intuition for DIME's two signals.} For a candidate $x^*$, $m_{\mathrm{local}}$ is the responsibility-weighted mean of nearby training points. An isolated member's $m_{\mathrm{local}}$ coincides with $x^*$ (low bias, low crowding); an isolated non-member's deviates substantially (high bias). A non-member inside a dense cluster can have $m_{\mathrm{local}}$ coincidentally close to $x^*$ (low bias) while the contributing points remain widely spread (high crowding, shaded ellipse).}
    \label{fig:dime-overview}
\end{figure*}

\paragraph{Contributions.} In summary, this work makes the following contributions:
\begin{itemize}
\item An exact, closed-form theoretical characterization of the
MSE-optimal denoiser for a diffusion model trained on a finite dataset (Section~\ref{sec:theory}), including an asymptotic isolation result explaining the behavior and the limitations of prior single-query attacks, and an exact bias-variance decomposition of a novel estimation error statistic.
\item \DIME, a query-efficient membership inference attack derived directly
from this decomposition, requiring only one reference model query and no generation of full samples, gradient access, or shadow-model training
(Algorithm~\ref{alg:dime}).
\item An extensive empirical evaluation across five diffusion checkpoints
spanning DDPM~\citep{ho2020denoising} and Guided Diffusion~\citep{DN21} architectures, demonstrating that \DIME outperforms existing baselines at matched or lower query cost across every
evaluated metric, and an evaluation of \DIME's robustness under DP-SGD (Section~\ref{sec:experiments}).
\end{itemize}
\section{Preliminaries and Background}
\label{sec:prelims}

\subsection{Notation}

We write $x\sim\mathbb{D}$ to denote $x$ a sample drawn from distribution $\mathbb{D}$. For some randomized mechanism $\mathcal{M}$, we also can denote $x\sim\mathcal{M}$ as an instance outputted by $\mathcal{M}$. For finite set $\mathcal{D}=\{x_1,\dots,x_n\}$, we write $x\leftarrow\mathcal{D}$ to denote $x$ assigned to a uniformly random element of set $\mathcal{D}$. We typically will elect to use a subscript $x_i$ to refer to an indexed dataset member, and superscript $x^*$ to refer to an attack point of inference to glean a membership decision from. We use superscript $x^{(t)}$ to denote a sample from a particular timestep in the diffusion process. We use capital letters (i.e. $X$) to denote random variables.

\subsection{Diffusion Models}

Diffusion models~\citep{song2019generative,song2020score,ho2020denoising} are a class of latent-variable generative models that represent data generation through a sequence of intermediate random variables, emerging as a leading approach for high-dimensional generative modeling, particularly in continuous domains such as images, audio, and scientific data.

We focus on the discrete DDPM formulation of diffusion models, closely following~\citep{ho2020denoising}. Supposing $x^{(0)}\in\mathbb{R}^d$ is a data sample drawn from unknown data distribution $\mathbb{D}$ to be learned, a diffusion model introduces latent random variables
$$X^{1:T}=(X^1,\dots,X^T),$$
so that the diffusion forward process repeatedly adds Gaussian noise such that
$$X^t|X^{t-1}\sim\mathcal{N}(\sqrt{1-\beta_t}X^{t-1},\beta_t I_d),$$
where for all $t\in[T]$ we have fixed constants $\beta_t\in(0,1)$ called the noise schedule. For notational convenience, we write $\alpha_t=1-\beta_t$ and $\overline\alpha_t=\alpha_1\cdots\alpha_t$. We might sometimes use $\sigma_t=\sqrt{1-\overline\alpha_t}$. From induction, we can also show that for any $t$ (see Equation 4 in~\citep{ho2020denoising}):
$$X^t|X^0\sim\mathcal{N}(\sqrt{\overline\alpha_t}X_0,(1-\overline\alpha_t)I_d).$$
Subsequently, joint (parameterized) distribution $p_\theta(x^{0:T})$ is called the reverse process and defined as a Markov chain with learned Gaussian transitions starting at $p_\theta(x^{(T)})\sim\mathcal{N}(0,I_d)$ and
$p_\theta(x^{(t-1)}|x^{(t)})\sim\mathcal{N}(\mu_\theta(x^{(t)},t), \Sigma_t),$
where
$$\mu_\theta(x^{(t)},t):=\frac{1}{\sqrt{\alpha_t}}\left(x^{(t)}-\frac{\beta_t}{\sqrt{1-\overline\alpha_t}}\hat\varepsilon_\theta(x^{(t)},t)\right),$$
$$\Sigma_t:=\frac{1-\overline\alpha_{t-1}}{1-\overline\alpha_t}\beta_t.$$
Here, $\hat\varepsilon_\theta(x^{(t)},t)$ is a trained noise prediction network that minimizes the expected noise for some $\mathcal{D}_{\mathsf{train}}\sim\mathbb{D}^*$:
\begin{equation}
    \mathcal{L}_t:=\mathbb{E}_{x\leftarrow\mathcal{D}_{\mathsf{train}} ,z\sim\mathcal{N}(0,I_d)}\left[\Vert\hat\varepsilon_\theta(\sqrt{\overline\alpha_t}x+\sigma_tz,t)-z\Vert^2\right].
    \label{eq:train-loss}
\end{equation}

\subsection{Membership Inference as Hypothesis Testing}

Given a model $f(\cdot;\theta)$ trained on a dataset $\mathcal{D}_{\mathsf{train}}$ and a
candidate point $x^*$, membership inference is the task of deciding
whether $x^*\in\mathcal{D}_{\mathsf{train}}$. A membership inference attack (MIA) is formalized as a (possibly
randomized) function $\mathcal{A}$ that, given $x^*$ and some level of
access to $f(\cdot;\theta)$, outputs a bit $\hat b\in\{0,1\}$, representing a guess as
to whether $x^*$ was a training member. In particular, membership inference implicitly invokes a classical hypothesis testing framework. Given some observations from a trained model, our goal is to distinguish between two statistical hypotheses:
\begin{itemize}
    \item $H_0: x^*\notin\mathcal{D}_{\mathsf{train}}$ (Null Hypothesis)
    \item $H_1: x^*\in\mathcal{D}_{\mathsf{train}}$ (Alternate Hypothesis)
\end{itemize}

Membership inference, across every domain in which it has been studied, reduces to a common methodological backbone: a \emph{test statistic} computed from a candidate point and the target model, thresholded to produce a binary membership decision. That is, given some test statistic $\mathcal{T}$ computed from observation, an induced membership inference adversary $\mathcal{A}$ with threshold $\tau$ outputs membership decision:
\begin{equation}
    \mathcal{A}=\mathbf{1}\left[\mathcal{T}\leq\tau\right].
    \label{eq:membership-inference}
\end{equation}
For any given test statistic $\mathcal{T}$, prior membership inference work  in~\citep{sablayrolles2019whitebox} suggests for the Bayes-optimal $\tau$ to be calibrated on a held-out validation set. For each test statistic evaluated in this work we follow that protocol for optimal performance while adhering to our specified threat model. We give the formal security game underlying this task, and the specific access $\mathcal{A}$ is granted in this work, in Section~\ref{sec:threat_model}.

\subsection{Threat Model}
\label{sec:threat_model}

We define the adversarial model $\Adv_{q}$ via security game, drawing from the framework presented in~\citep{carlini2022membership}.

\begin{definition}[Membership inference security game]
    The game proceeds between a challenger $\mathcal{C}$ and adversary $\mathcal{A}$
    \begin{enumerate}
        \item The challenger samples a training dataset $\mathcal{D}\sim\mathbb{D}^*$ from underlying data distribution $\mathbb{D}$, and trains a model $f(\cdot;\theta)\sim\mathcal{M}(\mathcal{D})$ on training set $\mathcal{D}$,
        \item The challenger flips a bit $b$, and if $b=0$ samples a fresh challenge point $x\sim\mathbb{D}$ (such that $x\notin\mathcal{D})$. Otherwise, the challenger selects a point from training set $x\leftarrow\mathcal{D}$.
        \item The challenger sends $x$ to the adversary.
        \item The adversary gets $q$ queries to the model $f(\cdot;\theta)$ and query access to the distribution $\mathbb{D}$, outputs a bit $\hat b\sim\mathcal{A}^{\mathbb{D},f}(x)$. We can think of $\mathsf{In}$ corresponding to bit $1$ and $\mathsf{Out}$ corresponding to bit $0$.
        \item Output $1$ if $\hat b=b$, and 0 otherwise.
    \end{enumerate}
    \label{def:mia-game-def}
\end{definition}

We specialize Definition~\ref{def:mia-game-def} to the diffusion setting by taking $f(\cdot;\theta)$ to be the trained noise-prediction network $\varepsilon_\theta(\cdot,\cdot)$, and by making explicit what a ``query'' means: the adversary submits a pair of noised input and timestep $(g,t)$ and observes $\epsilon_\theta(g,t)$. This is the same gray-box query interface assumed by every prior attack we compare against~\citep{rao2026scorebased,kong2024efficient,duan2023diffusion,matsumoto2023membership}, so that our comparison to prior work reflects a difference in test statistic and query efficiency, not a difference in assumed adversarial capability. By grey-box, we mean that the adversary has query access to $\epsilon_\theta(g,t)$ for adversary-chosen $(g,t)$, but no access to model weights, gradients, or internal activations, and no ability to backpropagate through the model.

\begin{example}
    Given candidate image $x^*$, the adversary chooses $t=10$, and sends $(x^*,t)$ to the model and receives the predicted noise vector $\hat\varepsilon_\theta(x^*,10)$. This entire vector counts as one query.
\end{example}

\begin{example}
    Given candidate image $x^*$, the adversary chooses $t=10$, samples 10 i.i.d. $z_k\sim\mathcal{N}(0,I_d)$, and constructs $g_k=\sqrt{\overline\alpha_t}x^*+\sigma_tz_k$. The adversary sends the 10 pairs $(g_k,t)$ to the model and receives 10 vectors of the form $\hat\varepsilon_\theta(g_k,t)$. This process counts as 10 queries.
\end{example}

We summarize our formal threat model below:

\paragraph{Envisioned attacker.}
The adversary in $\mathcal{A}_q$ has \emph{grey-box} access to the target model. The adversary additionally has query access to the underlying data distribution $\mathbb{D}$ (used to construct a held-out calibration set for selecting attack hyperparameters), but not to the training set $\mathcal{D}$ itself. 
That is, as is standard in empirical membership inference evaluation~\citep{rao2026scorebased,kong2024efficient,duan2023diffusion,matsumoto2023membership}, we allow the adversary auxiliary calibration data, generated independently
of the challenge bit, which may contain labeled examples of known members and
non-members of the target model. This information may be used to select hyperparameters and decision thresholds, but does not reveal the membership of the challenge point itself.
The adversary's goal is a binary decision: whether a specific, adversary-supplied challenge point $x$ was a member of $\mathcal{D}$.

\paragraph{Threat surface.}
The attack targets the \emph{deployment} phase of the ML lifecycle: it requires only query access to an already-trained model exposed via an inference interface (e.g., a hosted diffusion-model API, or any product surface that allows a caller to supply a noised input and timestep and receive the corresponding denoising output), and makes no assumptions about
and requires no access to the data curation or training phases. The attack surface is specifically the model's noise-prediction function as exposed at inference time, not the training pipeline or model weights at rest. 

\paragraph{Generality.}
The attack makes no assumption about model architecture, resolution, or conditioning beyond the standard diffusion forward/reverse process and
training objective; both the theoretical construction (Section~\ref{sec:theory}) and the resulting statistic apply universally across diffusion architecture. We demonstrate this empirically (Section~\ref{sec:experiments})
across five checkpoints spanning a $64\times$ range in input dimensionality and both unconditional and class-conditional settings, without any architecture-specific modification to the attack.

\paragraph{Practicality.}
The attack requires only $q+1$ total model queries: a single shared baseline query plus $q$ perturbed queries, each as a single forward pass through the denoising network, with no gradient computation, no model retraining, and critically, no generation of full samples via the reverse diffusion process. Our attack stands in contrast to shadow-model attacks~\citep{shokri2017membership} (requiring the adversary to train multiple auxiliary models) and
reconstruction-based attacks~\citep{pang2023black} (requiring full iterative sampling per query, orders of magnitude more expensive than a single forward pass). We show
(Section~\ref{sec:experiments}) that in some settings our attack achieves above SoTA results with as few as $2$ total queries, making it deployable against a
realistic, rate-limited or pay-per-query inference API where an adversary's query budget is a genuine, binding constraint.

\subsection{Evaluation of MIAs on Diffusion Models}

Because ground-truth membership is not observable for a model deployed in
practice, evaluating a membership inference attack requires a controlled
setting in which membership is known by construction. Following standard
practice, we partition a pool of candidate points into a \emph{member} set
(points drawn from $\mathcal{D}_{\mathsf{train}}$) and a \emph{held-out}
set (points drawn from the same underlying distribution $\mathbb{D}$ but
excluded from $\mathcal{D}_{\mathsf{train}}$), and report attack
performance over this labeled evaluation set. This labeling is an artifact
of evaluation, not a capability granted to the adversary in
Definition~\ref{def:mia-game-def}, who must still output $\hat b$ from
$x^*$ alone.
 
Given a test statistic $\mathcal{T}$ and this labeled evaluation set,
recall the induced decision rule in Equation~\eqref{eq:membership-inference}, i.e.\ a
\emph{smaller} value of $\mathcal{T}$ indicates membership.
Varying $\tau$
over all thresholds traces out a pair of rate functions,
\[
\mathrm{TPR}(\tau) := \Pr\big[\mathcal{T}\leq \tau \mid H_1\big], \qquad
\mathrm{FPR}(\tau) := \Pr\big[\mathcal{T}\leq \tau \mid H_0\big],
\]
estimated empirically over the labeled evaluation set. We report three
standard metrics derived from these rates throughout this work:
 
\begin{itemize}

\item \textbf{ASR} (attack success rate), the best balanced accuracy
achievable over all thresholds,
\(
\mathrm{ASR} := \max_\tau \tfrac{1}{2}\big(\mathrm{TPR}(\tau) + (1-\mathrm{FPR}(\tau))\big).
\)

\item \textbf{AUC}, the area under the resulting ROC curve
$(\mathrm{FPR}(\tau),\mathrm{TPR}(\tau))$, equal to the probability that
$\mathcal{T}$ assigns a lower value to a randomly chosen member than to a
randomly chosen non-member.
 
\item \textbf{TPR@1\%FPR}, the true-positive rate at the most permissive
threshold whose false-positive rate does not exceed $1\%$,
\(
\mathrm{TPR@1\%FPR} := \max\{\mathrm{TPR}(\tau) : \mathrm{FPR}(\tau)\leq 0.01\}.
\)
\end{itemize}
 
Unlike AUC and ASR, TPR@1\%FPR evaluates attack performance specifically in
the low-false-positive regime: the setting in which an adversary makes
confident, individually-defensible membership claims, which requires a low
false-positive rate regardless of average-case performance elsewhere on
the ROC curve~\citep{carlini2022membership}.
 
These metrics, and the evaluation protocol above, are not specific to any
model class; we apply them to the diffusion setting without modification,
consistent with the generality of the attack itself. What differs across the model checkpoints we evaluate is only the underlying population data distribution $\mathbb{D}$ and the resulting training member/held-out pools, described per
dataset in Section~\ref{sec:experiments}.
\section{Related Work}
\label{sec:related-work}

\subsection{Membership Inference}

Model inversion~\citep{fredrikson2014privacy,fredrikson2015model,carlini2021extracting,carlini2023extracting,zhang2020secret} and membership inference attacks (MIA) are a well studied privacy problem that predate the introduction of the diffusion process as a generative architecture. The notion of a test statistic for membership inference predates machine learning; Homer et al.~\citep{homer2008resolving} first demonstrated this kind of inference showing that an individual's presence in a pooled genomic mixture could be detected from a measurable statistical bias in comparing their known genotype against the mixture's aggregate allele frequencies. In the machine learning context, MIA was first defined for general supervised classification tasks~\citep{shokri2017membership} with shadow model training methods. Subsequently, the literature has moved to MIA for generative image models~\citep{chen2020gan} starting with generative adversarial networks (GANs) for its clear privacy and intellectual property implications.

\subsection{Membership Inference for Diffusion}

Originally, MIA in machine learning was studied in the ``black-box'' threat model, where attacks are performed without knowledge of the model weights or structure, relying only on input and output pairs from a deployed model. Prior work~\citep{sablayrolles2019whitebox} shows that under standard supervised classification, a single scalar is asymptotically sufficient for determining membership, that is more model access to intermediate-layer activations and gradients is uninformative as in a ``white-box'' threat model. The most common class of diffusion model MIAs, ``grey-box'' attacks span between these two extrema, generally having access to weights and/or internal representations, and crucially may query the model for particular test points. One of the first works in the diffusion MIA space \textit{Loss} takes the training loss as membership signal~\citep{matsumoto2023membership}. Similarly, follow-up works in \textit{SecMI} and \textit{PIA} use other diffusion quantities of reconstruction step error and ground-truth trajectories and observe practical improvements in performance~\citep{kong2024efficient,duan2023diffusion}. Finally, \textit{SimA} improves on prior work by introducing Monte Carlo sampling for attack stability, electing to use the score prediction as the test statistic~\citep{rao2026scorebased}. These prior methods are restricted to norm-based statistics, motivated heuristically or empirically. Our contribution goes beyond just another test statistic, we uniquely begin from the theoretical characterization of the query object (the denoiser) under empirical risk minimization and derive an overlooked membership relevant quantity in the crowding term.

\subsection{Misc. Diffusion Security}

Other adversarial work studies backdoor attacks on diffusion models, engineering a compromised training process so when a certain trigger is present, the model emits attacker-chosen content~\citep{chou2023baddiffusion}. We evaluate our attacks exclusively in the image generation domain, but prior work has investigated MIA in other data modalities, including text-to-image and text-to-video generation~\citep{shi2024detecting,li2024membership, hu2025membership}. Unlike our evaluated methods, these methods invoke the text modality and rely on token-level likelihood signals that are specific to auto-regressive generation. Prior work shows that a few poisoned samples can corrupt specific target prompt outputs in text-to-image diffusion generation~\citep{shan2024nightshade}. In contrast, our proposed method is tailored to diffusion models in the image-to-image generative domain and invokes the specific training loss objective. The theoretical underpinning to our metric is not directly applicable to cross-modalities, but we believe our methodology can be extended to detect training members of text-conditioned diffusion models.

\section{Membership Inference Frameworks: Algorithms and Theory}
\label{sec:theory}

In this section we first introduce and characterize the asymptotic diffusion denoiser to ground our new membership inference attack. Subsequently, from our theoretical analysis, we motivate and define a novel diffusion training test statistic which takes advantage of the unique membership error quantity we derive.

\subsection{Roadmap of the Framework}

A perfectly trained MSE denoiser computes a posterior-weighted average over possible training examples. For members, at low noise, this posterior concentrates on the member itself. This suggests measuring \emph{reconstruction error}. The reconstruction error has two pieces: displacement of the posterior mean (``bias'') and dispersion of training examples (``crowding''). The first is essentially what norm-based attacks~\citep{rao2026scorebased,duan2023diffusion,kong2024efficient, matsumoto2023membership} already see; the second is new. Both can be estimated from model queries, and their combination gives \DIME.

\subsection{Idealized Denoiser}

We term the \emph{idealized denoiser} as the theoretical function that minimizes the empirical diffusion training risk described in Equation~\ref{eq:train-loss}. We characterize an idealized denoiser $\hat\varepsilon_{\theta^*}$ that minimizes $\mathcal{L}_t$ defined in~\eqref{eq:train-loss} over all continuous functions described as $f:\mathbb{R}^d\times[T]\rightarrow\mathbb{R}^d$. The Universal Approximation Theorem~\citep{hornik1989multilayer,cybenko1989approximation,hornik1990universal} tells us that any continuous function can be approximated with arbitrary closeness by a neural network, meaning that given enough training over the appropriate loss objective, a learned model approaches the function described in Lemma~\ref{lem:optimal-denoiser}.

\begin{lemma}[Idealized Denoiser]
    Given a fixed training dataset $\mathcal{D}_{\mathrm{train}}=\{x_1,\dots,x_n\}$. The optimal denoiser $\hat\varepsilon^*$ that minimizes the loss:
    $$\mathcal{L}_t:=\mathbb{E}_{x\leftarrow\mathcal{D}_{\mathsf{train}} ,z\sim\mathcal{N}(0,I_d)}\left[\Vert\hat\varepsilon_\theta(\sqrt{\overline\alpha_t}x+\sigma_tz,t)-z\Vert^2\right]$$
    can be explicitly written as
    $$\hat\varepsilon^*(g,t)=\frac{g}{\sigma_t}-\frac{\sqrt{\overline\alpha_t}}{\sigma_t}\sum_{i=1}^n r_i(g)x_i,$$
    where $r_i(g)=\frac{s_i(g)}{\sum_{j=1}^ns_j(g)}$, and $s_i(g):=\exp(-\frac{1}{2\sigma_t^2}\Vert g-\sqrt{\overline\alpha_t}x_i\Vert ^2).$
    \label{lem:optimal-denoiser}
\end{lemma}

\begin{proof}
    We can write the loss function
    $$\mathcal{L}_t=\frac{1}{n}\sum_{i=1}^n\mathbb{E}_{z\sim\mathcal{N}(0,I)}[\Vert\hat\varepsilon_\theta(\sqrt{\overline\alpha_t}x_i+\sigma_tz,t)-z\Vert^2].$$
    Expanding the expectation, and writing $\phi(z)=(2\pi)^{-d/2}\exp(-\frac{1}{2}\Vert z\Vert^2)$ as the standard normal density, we have
    $$\frac{1}{n}\sum_{i=1}^n\int\Vert\hat\varepsilon_\theta(\sqrt{\overline\alpha_t}x_i+\sigma_tz,t)-z\Vert^2\phi(z)dz$$
    $$=\frac{1}{n}\sum_{i=1}^n\int\left\Vert\hat\varepsilon_\theta(g,t)-\left(\frac{g-\sqrt{\overline\alpha_t}x_i}{\sigma_t}\right)\right\Vert^2\phi\left(\frac{g-\sqrt{\overline\alpha_t}x_i}{\sigma_t}\right)\frac{dg}{\sigma_t^d}$$
    $$=\frac{1}{n\sigma_t^d}\int\sum_{i=1}^n\left\Vert\hat\varepsilon_\theta(g,t)-\left(\frac{g-\sqrt{\overline\alpha_t}x_i}{\sigma_t}\right)\right\Vert^2\phi\left(\frac{g-\sqrt{\overline\alpha_t}x_i}{\sigma_t}\right)dg,$$
    where the first equality follows from the change of variables
$g=\sqrt{\bar{\alpha}_t}x_i+\sigma_t z.$
Indeed,
$z=\frac{g-\sqrt{\bar{\alpha}_t}x_i}{\sigma_t},$
and since $g\in\mathbb{R}^d$, the Jacobian determinant of this
transformation is $\sigma_t^d$. Hence
$dz=\frac{1}{\sigma_t^d}\,dg.$
    Recall that we want to find the denoiser function $\hat\varepsilon^*(\cdot,t)$ that minimizes the loss $\mathcal{L}_t$, so we can minimize the integrand pointwise, which means that for any $g$, we want to minimize
    $\ell_g:=\sum_{i=1}^n\left\Vert\hat\varepsilon_\theta(g,t)-\left(\frac{g-\sqrt{\overline\alpha_t}x_i}{\sigma_t}\right)\right\Vert^2\phi\left(\frac{g-\sqrt{\overline\alpha_t}x_i}{\sigma_t}\right).$
    Clearly, this expression is quadratic, so setting $\frac{d\ell_g}{d\hat\varepsilon_\theta}=0$, we get
    $$\sum_{i=1}^n\left(\hat\varepsilon^*(g,t)-\frac{g-\sqrt{\overline\alpha_t}x_i}{\sigma_t}\right)\cdot\phi\left(\frac{g-\sqrt{\overline\alpha_t}x_i}{\sigma_t}\right)=0$$
    $$\implies \hat\varepsilon^*(g,t)\cdot\sum_{i=1}^ns_i(g)=\sum_{i=1}^n\frac{g-\sqrt{\overline\alpha_t}x_i}{\sigma_t}s_i(g)$$
    $$\implies \hat\varepsilon^*(g,t)=\frac{g}{\sigma_t}\sum_{i=1}^n r_i-\frac{\sqrt{\overline\alpha_t}}{\sigma_t}\sum_{i=1}^nr_i(g)x_i$$
    $$=\frac{g}{\sigma_t}-\frac{\sqrt{\overline\alpha_t}}{\sigma_t}\sum_{i=1}^nr_i(g)x_i,$$
    where the last equality comes from the fact that $\sum_{i=1}^n r_i=1$, which gives us the desired statement.
\end{proof}

We show that the idealized denoiser scales the input with an offset described by the posterior dataset ``responsibilities'' $r_i(g)$, which are precisely the posterior probability that training point $x_i$ was the source of the noisy observation $g$, under the same generative process the diffusion model was trained to invert. Structurally, $r_i(g)$ is a softmax over negative squared distances, acting as a smooth, probabilistic version of nearest-neighbor lookup. In other words, the optimal denoiser attends over the training set. We show that as the diffusion noise $\sigma_t\rightarrow 0$ asymptotically, the softmax becomes sharply peaked, and nearly all responsibility concentrates on whichever training point is closest to $g$.

\begin{definition}
    Fix some dataset $\mathcal{D}=\{x_1,\dots,x_n\}$ where $x_i\in\mathbb{R}^d$. Denote $d_i(x)=\Vert x-x_i\Vert$ for any $x\in\mathbb{R}^d$. Denote $V_i=\{x:d_i(x)=\min_j d_j(x)\}$.
\end{definition}

$V_i$ is exactly the Voronoi cell associated with training point $x_i$ in relation to the rest of $\mathcal{D}$.

\begin{lemma}[See Lemma~\ref{lem:convergence-denoiser-appendix}]
    Consider a sequence of functions 
    $r_i^{(j)}(g)=\frac{s_i^{(j)}(g)}{\sum_k s_k^{(j)}(g)},$
    where $s_i^{(j)}(g)=\exp\left(-\frac{1}{2(\sigma^{(j)})^2}\Vert g-a^{(j)}x_i\Vert^2\right)$, $a^{(j)}=\sqrt{1-\sigma^{(j)2}}$, and $\sigma^{(j)}\rightarrow 0^+$. Then
    $r_i^{(j)}(g)\longrightarrow\mathbf{1}_{V_i}(g)\quad\text{a.e.}$
    \label{lem:convergence-denoiser}
\end{lemma}

As the diffusion noise tends to zero, the soft posterior responsibilities become a hard nearest-neighbor rule.

Intuitively, Lemma~\ref{lem:optimal-denoiser} states that a perfectly-trained denoiser does not just guess noise in the abstract. For any noisy input, its best possible prediction is built by implicitly asking ``which training images could plausibly have produced
this?'' and blending their identities together, weighted by the posterior of how well each data member can produce the forward noised input. If the noisy input came from an image the model was actually
trained on, and the added noise is small, that image is overwhelmingly the best match: the model's responsibility blend collapses almost entirely onto it, and its prediction becomes highly accurate, nearly canceling out the true noise exactly. If the input instead came from an image the model has never seen,
there is no single, exact match to fall back on, but only an approximate blend of similar-looking training images, differing from the actual noise added to the non-member. This mismatch is not just larger on average for non-members; the same underlying construction lets us measure a second, complementary signal, particularly how much that blend of ``similar-looking'' images is spread out, or crowded together, around the query, giving two independent, theoretically grounded handles on membership. Lemma~\ref{lem:convergence-denoiser} shows that as the diffusion noise tends to zero, the idealized denoiser becomes a simple affine function of the input, precisely offset by the nearest dataset member, suggesting the best membership test lies in early timesteps. See Appendix~\ref{sec:appendix-experiments} for plots of the idealized denoiser on synthetic data. However, the idealized denoiser at low timesteps approaches jump discontinuity at Voronoi boundaries, which we find empirically to not be well captured by actual neural networks, resulting in a gap from our theoretical contribution, and our attacks tend to peak in performance early but not at the beginning of the diffusion timescale (see Figure~\ref{fig:stability}).

\subsection{Test Statistic}

The responsibility distribution $r(g)$ gives a direct, intuitive account of why detecting membership requires more than checking a single number as in prior attacks. Consider three cases, distinguished only by the local geometry of the training set around a candidate $x^*$.
 
\textbf{Case 1: an isolated member.} If $x^*\in\mathcal{D}$ and no other
training point lies nearby, responsibility concentrates almost entirely on
$x^*$ itself: $r_{x^*}(g)\to1$ as $t\to0$ (Lemma~\ref{lem:convergence-denoiser}).
The model's implicit reconstruction of the candidate's source is, in
expectation, exactly $x^*$, and because responsibility is concentrated on
a single point, there is essentially no disagreement across the (one)
candidate source either.
 
\textbf{Case 2: an isolated non-member.} If $x^*\notin\mathcal{D}$ but sits
near a single, otherwise-isolated training point $x_j\neq x^*$, the
posterior instead concentrates on $x_j$: $r_j(g)\to1$. Responsibility is
still sharply concentrated on one point, but the model's reconstruction is
now centered on the \emph{wrong} point, and it differs substantially from
$x^*$. This mismatch alone is enough to correctly flag $x^*$ as a
non-member, which is the basis of prior
diffusion membership inference attacks~\citep{rao2026scorebased,matsumoto2023membership,kong2024efficient,duan2023diffusion}.
 
\textbf{Case 3: a non-member inside a dense cluster -- the case that
motivates our construction.} Suppose instead $x^*\notin\mathcal{D}$, but several training points surround it closely and roughly symmetrically. No single point dominates the posterior, and responsibility spreads across the nearby cluster. Critically, the \emph{weighted average} of these several
nearby points can land, by simple geometric coincidence, very close to $x^*$ itself, even if $x^*$ matches none of them individually. The model's reconstruction therefore looks, at a glance, just as accurate as in the first case: a statistic that only checks how close the reconstruction lands
to $x^*$ cannot tell these two cases apart, which motivates our new attack method.

\paragraph{Membership as estimation error.}
The idealized denoiser derived above is, by construction, the function that most accurately recovers a training example from its noised observation. For a genuine training point $x_i$, this gives it a distinguishing
property: the denoiser's responsibility distribution $r(\sqrt{\overline\alpha_t}x_i)$ concentrates on $x_i$ itself in the small-noise limit (Lemma~\ref{lem:convergence-denoiser}), meaning the model's implicit \emph{reconstruction} of the identity of the
source point is, in expectation, exactly correct. For a point $x^*$ that was not
in the training set, no such guarantee exists; instead the model's best
reconstruction is some responsibility-weighted average of whichever training points happen to lie nearby, which do not necessarily coincide with $x^*$ itself.

This suggests a direct membership statistic: rather than testing attack point $x^*$'s
response against a hypothesized distribution, we measure \emph{how far the
idealized model's implicit reconstruction of $x^*$ deviates from $x^*$ itself}, e.g. an
estimation error. Members should exhibit small estimation error, whereas non-members derive larger error.

\paragraph{Defining the estimation error.}
Denote $I\sim r(g)$ the random variable indexing a training point drawn according to the model's responsibility distribution at $g=\sqrt{\overline\alpha_t}x^*$. That is, $I$ is defined as the random variable such that
$$\Pr(I=i|g)=r_i(g),$$
and let $Y=\sqrt{\bar\alpha_t}x_I$ be the corresponding (scaled) training point random variable. We
define the estimation error of the optimal denoiser at $x^*$ as
\begin{equation}
M(x^*,t) \;=\; \mathbb{E}\big[\|Y - \sqrt{\overline\alpha_t}x^*\|^2\big],
\label{eq:estimation_error}
\end{equation}
the expected squared distance between $x^*$ and a training point drawn in
proportion to how much the model currently attributes it to $x^*$. This is
a natural, direct measure of reconstruction quality: it is small precisely
when the model's implicit best guess at $x^*$'s identity is reliably close
to $x^*$ itself, and large when it is not.

\paragraph{Exact decomposition.}
The standard bias-variance identity~\cite{keener2010theoretical}
applied to $Y$ gives
\begin{equation}
M(x^*,t) = \underbrace{\|\mathbb{E}[Y]-\sqrt{\overline\alpha_t}x^*\|^2}_{\text{bias}^2}
\;+\; \underbrace{\mathrm{tr}(\mathrm{Cov}(Y))}_{\text{variance}}.
\label{eq:bias_variance}
\end{equation}

\begin{example}
    Consider a dataset with two points $\{-1,1\}$. For $x^*=0$, $Y=-1$ with probability $1/2$, and $Y=1$ with probability $1/2$, so $\mathbb{E}[Y]=0$ and so the bias is zero despite $x^*$ not being a training point. But the variance of $Y$ is 1. This motivating example necessitates the variance/crowding term.
\end{example}

From Corollary~\ref{cor:error-decomposition} in Appendix~\ref{sec:appendix-theory},
$\hat\varepsilon^*(g,t) = \frac{1}{\sigma_t}(g-\mathbb{E}[Y]),$
so the bias term is exactly $\sigma_t^2\|\hat\varepsilon^*(g,t)\|^2=:\sigma_t^2 B^2$,
where $B$ is the response magnitude of the optimal denoiser at $x^*$. The variance term is $\mathrm{tr}(C)=:V$, where
$C=\mathrm{Cov}(Y)=\mathrm{Cov}_{I\sim r(g)}[\sqrt{\overline\alpha_t}x_I],$ 
which is the expected squared distance of a responsibility-sharing training point
from the weighted centroid of its neighborhood, which we show has a closed-form estimator via
the Hutchinson trace estimator~\citep{hutchinson1989stochastic} applied to $\nabla_g\hat\varepsilon^*(g,t)$.
Thus
\begin{equation}
M(x^*,t) = \sigma_t^2 B^2 + V,
\label{eq:M_decomposition}
\end{equation}
exactly; the estimation error decomposes, with no approximation, into two
independently estimable and individually meaningful quantities: 1. how far
the average nearby training point sits from $x^*$, and 2. how spread out the nearby training points are around one another. Membership manifests in
both: an isolated member drives both terms toward zero as $t\to0$
(Lemma~\ref{lem:convergence-denoiser}), and we find that a non-member's estimation error can be
large through either channel, from a poorly-reconstructed average location and/or a
crowded neighborhood (See Figure~\ref{fig:bias-variance}).

\paragraph{Finite-sample estimators.}
Neither $B$ nor $V$ is directly observable; both are defined in terms of an expectation over the full training set. We estimate each using only forward evaluations of $\hat\epsilon^*$ at $q$ randomly perturbed points around $g=\sqrt{\overline\alpha_t}x^*$. For step size $\gamma>0$ and $z_1,\dots,z_q\overset{\text{iid}}{\sim}\{-1,+1\}^d$ (Rademacher), we can define $\hat B$ (estimator of $B$) as:
\begin{equation}
\hat B(x^*,t,\gamma;q) \;=\; \frac{1}{q}\sum_{k=1}^q \big\|\hat\varepsilon_\theta(g+\gamma\sigma_t z_k,\,t)\big\|,
\label{eq:hat_B}
\end{equation}

a Monte Carlo estimate of the expected response magnitude at radius $\gamma$. From Lemma~\ref{lem:V-error-bias} in Appendix~\ref{sec:appendix-theory}, we can define $\hat V$ estimating an affine transformation of $V$ using a Hutchinson trace estimator~\citep{hutchinson1989stochastic} of the gradient of the idealized denoiser $\mathrm{tr}(\nabla_g\hat\varepsilon^*(g,t))$, incurring only $O(\gamma)$ estimator bias. We define explicitly
\begin{equation}
\hat V(x^*,t,\gamma;q) \;=\; \frac{1}{q}\sum_{k=1}^q \frac{z_k^\top\big(\hat\varepsilon_\theta(g+\gamma\sigma_tz_k,\,t)-\hat\varepsilon_\theta(g,t)\big)}{\gamma\sigma_t}.
\label{eq:hat_V}
\end{equation}
 
Both estimators share the same $q$
perturbed queries and a single base query at $g$, for a total query cost of $q+1$ per evaluated $(t,\gamma)$. See Lemmas~\ref{lem:squared-bias-estimator} and~\ref{lem:V-error-bias} in Appendix~\ref{sec:appendix-theory} for more detailed analysis of our finite-sample estimators of the membership-error bias-variance decomposition.

\paragraph{From the ideal quantity to a practical statistic.}
Defined above are finite-sample estimators $\hat B$ and
$\hat V$ of $B$ and $V$ in Equation~\eqref{eq:M_decomposition},
each requiring only forward evaluations of $\hat\varepsilon_\theta$. We find empirically that allowing a free weight hyperparameter $\phi$ which tunes the relative weight of the two estimation error components improves attack performance at low tuning cost; in practice we calibrate the relative weighting of the two error components directly against held-out data, remaining in line with the adversarial model's knowledge of the data distribution (Definition~\ref{def:mia-game-def}), rather than fixing it at its population-exact value, which finally yields \DIME (Denoiser Ideal Membership Error):
\begin{equation}
\mathrm{\DIME}(x,t,\gamma,\phi) = \cos\phi\cdot\hat B^2 + \sin\phi\cdot\hat V,
\label{eq:DIME}
\end{equation}
the calibration-weighted combination of the two estimation error components of the ideal
estimation error, with $t$, $\gamma$, and $\phi$ selected by grid search on a held-out split. Because $(\cos\phi,\sin\phi)$ is the standard parameterization of every point on the unit circle and only the direction (ratio), affects test statistic performance, restricting the combination $w_1\hat B+w_2\hat V$ to unit-norm weights loses no achievable ranking while reducing the search from  two free parameters down to one. Unlike the fixed-weight construction, this combination adapts to standard estimation error as well as which error component is, in practice, more reliably estimable at the query budget available, while remaining a direct estimator of the same quantity: how far the optimal denoiser's implicit reconstruction of $x^*$ deviates from $x^*$ itself.


\begin{algorithm}[t]
\DontPrintSemicolon
\caption{\DIME: Denoiser Ideal Membership Error}
\label{alg:dime}
\KwIn{candidate $x^*$; timestep $t$; perturbation radius $\gamma$; query
budget $q$; calibrated combination angle $\phi$; calibrated threshold
$\tau$}
\KwOut{membership decision $\hat b \in \{0,1\}$}

$g \gets \sqrt{\bar\alpha_t}\, x^*$\;
$\epsilon_{\text{base}} \gets \hat\epsilon_\theta(g, t)$ \tcp*{1 query}
$\text{bias\_sum} \gets 0$; $\text{crowd\_sum} \gets 0$\;
\For{$k = 1$ \KwTo $q$}{
    Draw $z_k \sim \mathrm{Rademacher}(\{-1,+1\}^d)$\;
    $g_k \gets g + \gamma \sigma_t z_k$\;
    $\epsilon_k \gets \hat\epsilon_\theta(g_k, t)$ \tcp*{1 query}
    $\text{bias\_sum} \gets \text{bias\_sum} + \|\epsilon_k\|$\;
    $\text{crowd\_sum} \gets \text{crowd\_sum} + \dfrac{z_k^\top(\epsilon_k - \epsilon_{\text{base}})}{\gamma \sigma_t}$\;
}
$\hat B \gets \text{bias\_sum} / q$\;
$\hat V \gets \text{crowd\_sum} / q$\;
$\mathrm{score} \gets \cos \phi\cdot \hat B^2 + \sin \phi\cdot \hat V$\; 
$\hat b \gets \mathbf{1}[\mathrm{score} \leq \tau]$\;
\Return{$\hat b$}
\end{algorithm}

We provide a precise algorithm of our \DIME method in Algorithm~\ref{alg:dime}. As a brief aside, our implementation standardizes the $\hat{B},\hat{V}$ terms in hyperparameter calibration search for numerical convenience. With an explicit definition, it remains to evaluate the effectiveness of our new attack on diffusion membership inference. In the next section, we design and run experiments to quantify the membership signal contained in our idealized denoising estimation error test statistic.

\section{Experiments: Setup and Evaluation}
\label{sec:experiments}

Although membership inference has been extensively studied for discriminative and generative models, diffusion models have a fundamentally different iterative denoising process and query interface, which may affect both the design and effectiveness of membership inference attacks. Understanding these differences is important for accurately assessing privacy risks in diffusion models. Our work on the idealized denoiser (Lemma~\ref{lem:optimal-denoiser}) characterizes the model outputs, raising specific testable questions about both when existing attacks should be expected to succeed or fail, and what additional structure an idealized attack could exploit that prior heuristics do not. We aim to empirically answer the following research questions:

\noindent{\textbf{RQ1.}} Our attack introduces a membership error variance crowding term which captures the denoiser's local structure around an attack point, which has no analog in prior membership inference literature. Does a test statistic grounded in idealized membership error carry membership-relevant information beyond previous simple query-based methods? 

\noindent{\textbf{RQ2.}} How does the performance of diffusion-model membership inference attacks change under different query budgets, and which attack provides the best trade-off between privacy leakage detection and query cost (e.g., inference time or computation) when the number of queries is limited?

\noindent{\textbf{RQ3.}} The theoretical characterization makes no assumption about
resolution, conditioning, dataset structure, or training procedure. Does the resulting attack's advantage over prior methods persist when these
assumptions are stressed, i.e. at substantially higher resolution and under class-conditioning (Guided Diffusion / ImageNet), under non-standard dataset geometry (STL10-U, CelebA), and under a training procedure with a formal privacy guarantee (DP-SGD)?

As a quick summary, we find the answer to all three research questions is yes. On \textbf{RQ1}, our \DIME test statistic carries substantially more signal than any single-query statistic in prior work;
at matched query cost, our method improves TPR@1\%FPR by up to $3\times$ over the strongest baseline (SimA-MC) on CIFAR-10 (50.8\% vs.\ 16.1\%). On \textbf{RQ2}, this advantage holds even at minimal query budgets: our cheapest setting (2 queries) already matches or exceeds every baseline's \emph{most expensive} setting (30 queries) on multiple datasets, including a $1.5\times$
TPR@1\%FPR improvement over SimA-MC's best multi-query result on ImageNet (14.3\% vs.\ 9.33\%) at a fraction of the query cost. On \textbf{RQ3}, this advantage persists under substantial architectural and distributional stress, maintaining the best per-query performance across dataset and model fidelity while
disappearing entirely under DP-SGD: our method, like every baseline, collapses to nearly chance-level performance (AUC within $48$-$53\%$) at every tested privacy budget, confirming that DP-SGD remains an effective defense even against our strongest attack evaluated in this work.

\subsection{Datasets and Checkpoints}

We evaluate previous baselines and our attack over five benchmark datasets spanning a deliberate range of scale and domain, namely CIFAR-10/100~\citep{krizhevsky2009learning}, STL10-U~\citep{coates2011analysis}, CelebA~\citep{liu2015deep}, and ImageNet-1k~\citep{russakovsky2015imagenet}. We briefly describe how model weights and data splits for each dataset are obtained for reproducibility. All experiments are done on a single NVIDIA A40 GPU.

\begin{itemize}
    \item \textbf{CIFAR-10/CIFAR-100:} We reuse the public model checkpoints and dataset member/held-out splits (25k/25k) drawn from~\citep{duan2023diffusion}.
    \item \textbf{STL10-U:} We trained a model on a single NVIDIA A40 GPU node from scratch for 18k training steps with 10k/10k member/heldout splits.
    \item \textbf{CelebA:} We trained a model from scratch on a single NVIDIA A40 GPU node for 60k training steps with 30k/30k member/heldout splits.
    \item \textbf{ImageNet-1k:} Similarly to prior literature~\citep{rao2026scorebased}, we reuse public Guided Diffusion model~\citep{DN21} checkpoints trained on ImageNet-1k, and use the validation set ImageNetV2~\citep{recht2019imagenet} as the held-out set.
\end{itemize}

\paragraph{Baselines.} We use prior membership inference attack baselines \emph{Loss}~\citep{matsumoto2023membership}, \emph{PIA}~\citep{kong2024efficient}, \emph{SecMI}~\citep{duan2023diffusion}, \emph{SimA / SimA-MC}~\citep{rao2026scorebased}. These attacks compute test statistics based on quantities in the diffusion process such as diffusion trajectories, estimated score norm, and reconstruction error, in comparison to our idealized denoiser-based attack.
We define their test statistics below. All tests are based on tunable hyperparameter thresholds for determining  membership. These baselines are chosen in particular for their SoTA performance in diffusion training membership inference, across varying query budgets~\citep{rao2026scorebased}.

In each of the methods below, $\varepsilon$ is independently sampled standard Gaussian noise used to compute the test statistic. The Loss criterion proposed in~\citep{matsumoto2023membership} uses the training loss function, or:
$$\mathcal{T}_{Loss}(x,t)=\Vert\varepsilon-\hat\varepsilon_\theta(\sqrt{\overline\alpha_t}x+\sigma_t\varepsilon, t))\Vert$$

SecMI of~\citep{duan2023diffusion} takes instead the reconstruction error of a noising and denoising step, taking the following statistic:
$$\mathcal{T}_{Sec}(x,t)=\Vert \sqrt{1-\overline\alpha_t}(\hat\varepsilon_\theta(x,t)-\hat\varepsilon_\theta(\sqrt{\overline\alpha_{t+1}}x+\sigma_{t+1}\varepsilon,t+1))\Vert$$

PIA introduced by~\citep{kong2024efficient} uses the ground-truth diffusion trajectory to determine membership, using:
$$\mathcal{T}_{PIA}(x,t)=\Vert\hat\varepsilon_\theta(x,0)-\hat\varepsilon_\theta(\sqrt{\overline\alpha_t}x+\sigma_t\hat\varepsilon_\theta(x,0),t)\Vert$$

Finally, SimA~\citep{rao2026scorebased} uses the norm of the predicted score as the decision criterion, that is:
$$\mathcal{T}_{SimA}(x,t)=\Vert\hat\varepsilon_\theta(x,t)\Vert_p$$
and we also evaluate a more stable Monte-Carlo variant
$$\mathcal{T}_{SimA}(x,t)=\frac{1}{N}\sum_n\Vert\hat\varepsilon_\theta(x_t,t)\Vert,\; x_t=\sqrt{\overline\alpha_t}x+\sigma_t\varepsilon_n,$$
where $\varepsilon_n\sim\mathcal{N}(0,I)$.
We re-implement all of the methods for experimental consistency, drawing mainly from the authors' provided codebases.

\paragraph{Evaluation protocol.} Member/held-out sets are split into disjoint calibration and test halves. All test hyperparameters (including timestep selection) are selected by maximizing the AUC on the calibration half only; all reported metrics are computed once on the held-out test half, which is a stricter protocol than in prior work~\citep{rao2026scorebased}.

\paragraph{Metrics.} We use the following metrics in line with prior literature: ASR (attack success rate), AUC (Area Under ROC Curve), TPR@1\% FPR (fraction of true members correctly identified as members where at most 1\% of non-members are misclassified), which was introduced in~\citep{carlini2022membership}. We also note the number of model queries each attack uses in our tables, especially pertinent to SimA and our own attacks which improve with more model queries.

\paragraph{Model architectures.}
For CIFAR-10, CIFAR-100, STL10-U, and CelebA, we use the standard DDPM U-Net
architecture of~\cite{ho2020denoising} (base channels 128, channel
multipliers $[1,2,2,2]$, 2 residual blocks per resolution, self-attention at
$16{\times}16$, GroupNorm normalization and Swish activations throughout),
the same architecture used by SecMI~\cite{duan2023diffusion},
PIA~\cite{kong2024efficient}, and SimA~\citep{rao2026scorebased} in their respective evaluations. This model
is \emph{unconditional}: it predicts noise given only a noised image and
timestep, at $32{\times}32$ resolution ($d\approx3{,}072$).
 
For ImageNet-1k, we use Guided Diffusion~\cite{DN21}, a
substantially larger, publicly released checkpoint operating at $256{\times}256$ resolution ($d\approx196{,}608$, a $64\times$ increase in
input dimensionality). Unlike our DDPM checkpoints, Guided Diffusion is \emph{class-conditional}: the class label is injected as an additional conditioning signal alongside the timestep embedding, and must be supplied
at every query. Both architectures instantiate the same underlying forward diffusion process and noise-prediction training objective.

\subsection{Results and Discussion}

\begin{table*}[t]
\centering
\caption{Performance of MIA attacks on DDPM across four benchmark datasets. \DIME outperforms all other methods in all metrics at comparable query costs, with the biggest gains in TPR@1\%FPR.}
\label{tab:main_results}
\resizebox{\textwidth}{!}{%
\begin{tabular}{l c ccc ccc ccc ccc}
\toprule
& & \multicolumn{3}{c}{CIFAR-10 (\%)} & \multicolumn{3}{c}{CIFAR-100 (\%)} & \multicolumn{3}{c}{STL10-U (\%)} & \multicolumn{3}{c}{CelebA (\%)} \\
\cmidrule(lr){3-5} \cmidrule(lr){6-8} \cmidrule(lr){9-11} \cmidrule(lr){12-14}
Method & \#Query$\downarrow$ & ASR$\uparrow$ & AUC$\uparrow$ & TPR@1\%FPR$\uparrow$
       & ASR$\uparrow$ & AUC$\uparrow$ & TPR@1\%FPR$\uparrow$
       & ASR$\uparrow$ & AUC$\uparrow$ & TPR@1\%FPR$\uparrow$
       & ASR$\uparrow$ & AUC$\uparrow$ & TPR@1\%FPR$\uparrow$ \\
\midrule
Loss                    & 1     & 77.30 & 84.38 & 6.59 & 75.31 & 82.51 & 9.44 & 93.85 & 97.95 & 52.14 & 67.32 & 73.50 & 5.67 \\
PIA                     & 2     & 81.67 & 88.58 & 14.20 & 79.14 & 86.34 & 15.26 & 95.28 & 98.77 & 70.54 & 72.63 & 80.06 & 9.75 \\
SecMI\textsubscript{stat} & $\sim$12 & 81.45 & 88.51 & 10.45 & 80.85 & 88.05 & 13.34  & 93.98 & 98.20 & 57.76 & 71.08 & 78.03 & 8.64 \\
SimA ($\ell_4$)         & 1     & 83.60 & 90.43 & 14.07 & 82.40 & 89.40 & 13.64 & 95.42 & 98.61 & 62.68 & 71.09 & 78.21 & 8.91 \\
SimA-MC ($\ell_4$, \#mc=10) & 10 & 84.26 & 91.34 & 16.95 & 83.22 & 90.55 & 19.03 & 97.25 & 99.31 & 80.92 & 74.33 & 82.05 & 9.91 \\
SimA-MC ($\ell_4$, \#mc=30) & 30 & 85.15 & 92.00 & 16.14 & 83.74 & 90.99 & 18.87 & 96.74 & 99.23 & 79.84 & 75.03 & 82.96 & 12.91 \\
\midrule
\DIME (Ours) & 1+1 & \textbf{84.18} & \textbf{91.50} & \textbf{30.88}
& \textbf{83.00} & \textbf{90.02} & \textbf{22.79} & \textbf{96.87} & \textbf{99.39} & \textbf{86.16} & \textbf{73.54} & \textbf{81.02} & \textbf{11.63} \\
\DIME (Ours, mc=10), & 10+1 & \textbf{86.43} & \textbf{93.35} & \textbf{47.23} & \textbf{85.01} & \textbf{91.75} & \textbf{26.93} & \textbf{98.65} & \textbf{99.78} & \textbf{97.78} & \textbf{76.67} & \textbf{84.55} & \textbf{14.08} \\
\DIME (Ours, mc=30), & 30+1 & \textbf{87.30} & \textbf{93.89} & \textbf{50.77} & \textbf{85.20} & \textbf{91.92} & \textbf{24.83} & \textbf{98.80} & \textbf{99.80} & \textbf{98.00} & \textbf{78.14} & \textbf{86.31} & \textbf{18.78} \\
\bottomrule
\end{tabular}%
}
\end{table*}


\begin{table}[t]
\centering
\caption{Performance on ImageNet-1k / Guided Diffusion (256x256, class-conditional). \DIME outperforms all other methods in all metrics at comparable query costs, with the biggest gains in TPR@1\%FPR.}
\label{tab:imagenet_results}
\resizebox{\linewidth}{!}{%
\begin{tabular}{l c ccc}
\toprule
Method & \#Query$\downarrow$ & ASR$\uparrow$ & AUC$\uparrow$ & TPR@1\%FPR$\uparrow$ \\
\midrule
Loss                        & 1        & 74.62 & 80.64 & 1.60 \\
PIA                         & 2        & 67.56 & 69.88 & 1.07 \\
SecMI\textsubscript{stat}   & $\sim$12 & 66.67 & 50.59 & 0.93 \\
SimA ($\ell_4$)             & 1        & 83.84 & 86.35 & 7.20 \\
SimA-MC ($\ell_4$, mc=10)   & 10       & 84.13 & 89.41 & 9.33 \\
SimA-MC ($\ell_4$, mc=30) & 30 & 82.58 & 87.45 & 1.87 \\
\midrule
\DIME (Ours)    & 1+1  & \textbf{86.45} & \textbf{92.50} & \textbf{14.27} \\
\DIME (Ours, mc=10) & 10+1 & \textbf{86.96} & \textbf{92.58} & \textbf{15.27} \\
\DIME (Ours, mc=30) & 30+1 & \textbf{87.33} & \textbf{92.65} & \textbf{15.27} \\
\bottomrule
\end{tabular}%
}
\end{table}

\noindent{\textbf{DDPM.}} Table~\ref{tab:main_results} reports \DIME against Loss, PIA, SecMI$_\text{stat}$, SimA, and SimA-MC across four DDPM checkpoints (CIFAR-10, CIFAR-100, CelebA, STL10-U). We compare at matched, or near-matched, query budgets.
 
At every dataset and every matched tier, \DIME outperforms the strongest baseline on all evaluation metrics simultaneously. The gap is most pronounced on TPR@1\%FPR, the metric prior work~\cite{carlini2022membership} has argued is most operationally meaningful for membership inference: at 11 queries, \DIME achieves 47.23\% TPR@1\%FPR on CIFAR-10, compared to SimA-MC's best result of 16.14\% at 30 queries, displaying roughly $3\times$ the detection rate at a third of the query cost. The same pattern holds on CelebA (14.08\% at 11 queries vs. 12.91\% at the highest tier for SimA) and is most dramatic on STL10-U, where \DIME reaches 97.44\% TPR@1\%FPR at 11 queries against SimA-MC's best 80.92\%, which is near-total detection. On CIFAR-100, \DIME improves TPR@1\%FPR from SimA-MC's 19.03\% to 26.93\%; we note a modest, isolated dip in TPR@1\%FPR at the highest query tier on this dataset not mirrored in AUC or ASR (both of which continue to improve), consistent with the greater sensitivity of tail-probability estimates to sampling noise at this scale, and report the 11-query tier as CIFAR-100's representative operating point.
 
These improvements hold consistently whether comparing against SimA-MC (and other methods) at matched query count or, more conservatively, against its best result at any query count in our evaluation; we observe that \DIME's cheapest tier alone (2 queries) already exceeds every baseline's TPR@1\%FPR on CIFAR-10 and CIFAR-100.

\noindent{\textbf{Guided-Diffusion.}} To test whether these results generalize beyond the 32$\times$32 DDPM
checkpoints above, we evaluate \DIME against Guided
Diffusion~\cite{DN21}, a substantially larger,
256$\times$256, class-conditional model, using a stratified sample of
ImageNet-1k (member) and ImageNetV2~\cite{recht2019imagenet} (held-out).
This setting differs from our DDPM checkpoints not only in resolution and
parameter count, but in architecture (class-conditioning) and input
dimensionality ($d\approx196{,}608$, a 64$\times$ increase over CIFAR-scale
inputs), which is a meaningfully more demanding test of whether our theoretical construction transfers beyond the setting it was developed in.
 
Table~\ref{tab:imagenet_results} reports the resulting comparison. \DIME outperforms every baseline on ASR, AUC, and TPR@1\%FPR simultaneously, at every query budget we evaluate. Even at its cheapest setting (2 queries), \DIME's TPR@1\%FPR (14.27\%) exceeds every baseline's best result at any query cost, including SimA-MC's best 10-query setting (9.33\%), showing a $1.5\times$ improvement at roughly $1/5$th the budget. At a directly comparable cost (11 queries vs. SimA-MC's 10), \DIME wins all three metrics outright: 86.96\% ASR vs. 84.13\%, 92.58\% AUC vs. 89.41\%, and 15.27\%
TPR@1\%FPR vs. 9.33\%.
 
We note two additional observations specific to this setting. First, \DIME's TPR@1\%FPR plateaus between 11 and 31 queries (15.27\% at both), while ASR and AUC continue to improve marginally (86.96\%$\to$87.33\%,
92.58\%$\to$92.65\%), which is unlike the clean, monotonic gains from additional queries observed on our DDPM checkpoints, suggesting the marginal value of
additional averaging diminishes faster at this scale, plausibly reflecting the smaller evaluation set available at ImageNet's higher per-query cost.
Second, SimA-MC itself is non-monotonic in query count here: its TPR@1\%FPR
at mc=30 (1.87\%) is substantially \emph{lower} than at mc=10 (9.33\%), indicating this instability is not specific to our method but reflects a broader difficulty in reliably estimating tail-probability statistics at this resolution.
 
Taken together, these results indicate that the theoretical grounding motivating \DIME, derived from the exact optimal-denoiser solution for a
finite training set (Lemma~\ref{lem:optimal-denoiser}), yields consistent practical improvements across a substantial architectural and scale gap, not only on the smaller checkpoints used in our primary evaluation. Illustratively, Figure~\ref{fig:split-distributions} shows the clear statistical separation between member and held-out data points in the \DIME statistic across all datasets and trained checkpoints.

Although we report the best calibrated timestep $t$ in our main tables, we plot the robustness of \DIME to varying timesteps in Figure~\ref{fig:stability}. Lemma~\ref{lem:convergence-denoiser} prescribes a stronger membership signal for earlier timesteps in the idealized denoiser, which is reflected strongly in the Guided Diffusion ImageNet experiment, but not in the DDPM experiments on the other datasets. We explain this discrepancy simply. While our DDPM model checkpoints use $\sim 35$M parameters, the Guided Diffusion checkpoint substantially expands with $\sim 550$M parameters, suggesting that the heavier architecture with its additional expressivity conforms better to our theoretical object in the idealized denoiser.

Our experiments demonstrate that our optimal denoiser-based membership inference attacks outperform previously considered methods across a variety of dataset fidelities, model architectures, and evaluation metrics.

\begin{figure*}[t]
    \centering
    \includegraphics[width=\textwidth]{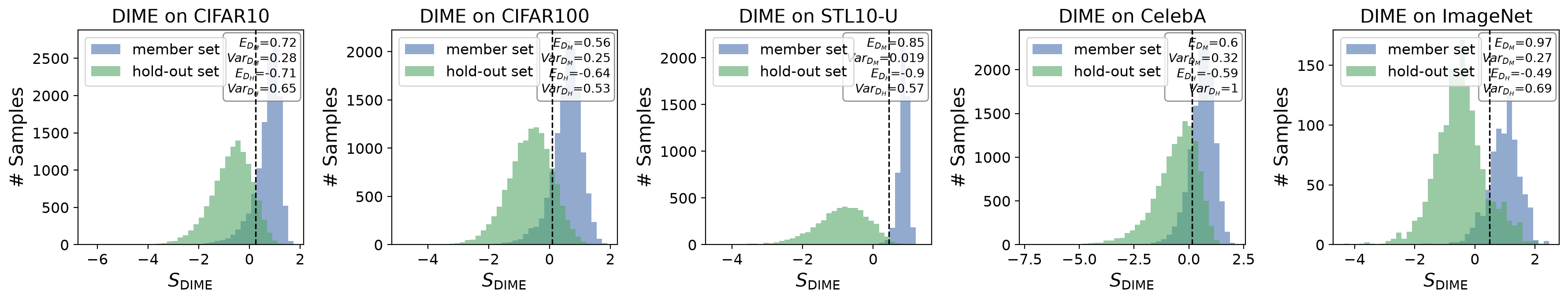}
    \caption{\DIME test statistic distributions for samples from member and held-out sets. The vertical dashed line denotes the selected threshold $\tau$ for each figure.}
    \label{fig:split-distributions}
\end{figure*}

\begin{figure*}[t]
    \centering
    \includegraphics[width=\textwidth]{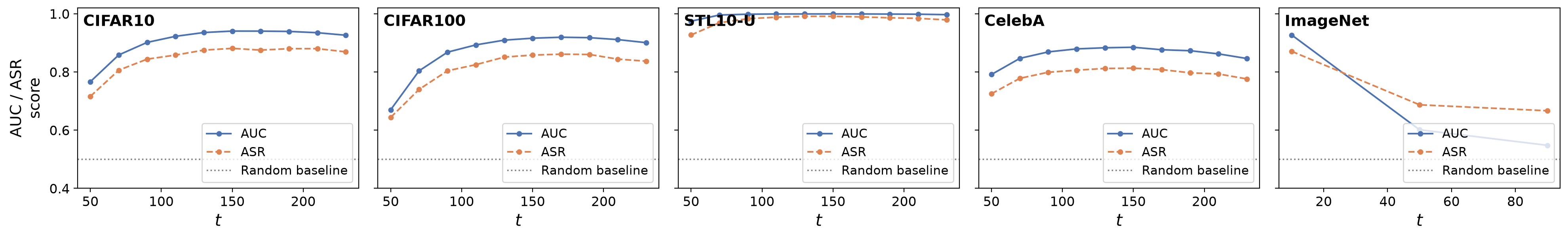}
    \caption{Robustness of the \DIME test to timestep across five datasets. DDPM models remain stable across all time-steps, while the Guided Diffusion model is most vulnerable to membership inference at early timesteps.}
    \label{fig:stability}
\end{figure*}

\section{Defenses}
\label{sec:defense}


\begin{table*}[tb]
\centering
\caption{Attack accuracy under DP-SGD defense (CelebA, $n{=}1000$, identity-aware
split, 500 epochs). All five attack methods' accuracy collapses under DP-SGD,
regardless of $\epsilon$. `Vanilla' means without DP-SGD. \DIME's \#Query is
$q_{\text{ref}}{+}1$ (11 at $q_{\text{ref}}{=}10$, the reported tier, roughly
comparable to SimA-MC's 10). The highest accuracy in each column is bolded.}
\label{tab:dpsgd_results_n1000}
\resizebox{\textwidth}{!}{%
\begin{tabular}{l ccc ccc ccc ccc ccc}
\toprule
& \multicolumn{3}{c}{Loss} & \multicolumn{3}{c}{PIA} & \multicolumn{3}{c}{SecMI$_{\text{stat}}$}
& \multicolumn{3}{c}{SimA ($\ell_4$)} & \multicolumn{3}{c}{\DIME (Ours)} \\
\cmidrule(lr){2-4} \cmidrule(lr){5-7} \cmidrule(lr){8-10} \cmidrule(lr){11-13} \cmidrule(lr){14-16}
& ASR$\uparrow$ & AUC$\uparrow$ & T@1\%F$\uparrow$
& ASR$\uparrow$ & AUC$\uparrow$ & T@1\%F$\uparrow$
& ASR$\uparrow$ & AUC$\uparrow$ & T@1\%F$\uparrow$
& ASR$\uparrow$ & AUC$\uparrow$ & T@1\%F$\uparrow$
& ASR$\uparrow$ & AUC$\uparrow$ & T@1\%F$\uparrow$ \\
\midrule
$\epsilon=1$  & 0.530 & 0.521 & 1.00 & 0.526 & 0.516 & 1.60 & 0.517 & 0.500 & 1.80 & 0.517 & 0.505 & 0.40 & 0.529 & 0.517 & 0.40 \\
$\epsilon=4$  & 0.539 & 0.518 & 1.40 & 0.506 & 0.481 & 0.60 & 0.522 & 0.488 & 1.20 & 0.515 & 0.502 & 0.20 & 0.525 & 0.515 & 1.20 \\
$\epsilon=10$ & 0.519 & 0.509 & 1.40 & 0.506 & 0.483 & 0.80 & 0.523 & 0.519 & 1.00 & 0.510 & 0.489 & 1.00 & 0.537 & 0.529 & 0.80 \\
Vanilla       & \textbf{0.690} & \textbf{0.754} & \textbf{10.20} & \textbf{0.753} & \textbf{0.829} & \textbf{15.40} & \textbf{0.710} & \textbf{0.777} & \textbf{13.60} & \textbf{0.745} & \textbf{0.827} & \textbf{16.20} & \textbf{0.798} & \textbf{0.869} & \textbf{22.40} \\
\bottomrule
\end{tabular}%
}
\end{table*}

Differential privacy (DP)~\cite{DworkMNS06,DworkKMMN06} is the standard formal defense against membership inference: Differentially Private Stochastic Gradient Descent (DP-SGD)~\cite{abadi2016deep} is a standard machine learning algorithm that implements differential privacy. 
It bounds the influence that any single training example can have on the final model, giving a provable guarantee independent of the specific attack used against it. Prior work evaluating diffusion-model memorization under DP-SGD finds that existing attacks are rendered ineffective at standard privacy budgets on small-scale fine-tuning datasets~\citep{pang2023black}. We evaluate whether this holds against \DIME as well: a defense that only defeats weaker, prior attacks is a substantially less interesting result than one that also defeats a stronger, theoretically-motivated one.

\subsection{Differential Privacy: Definition and Implications}

One of the most ubiquitous definitions of privacy due to its rigorous guarantees, differential privacy, referred to as DP~\citep{dwork2006differential,DworkKMMN06}, quantifies the worst-case privacy risk admitted by randomized algorithms.

\begin{definition}[Differential Privacy]
    A (randomized) mechanism $\mathcal{M}$ satisfies $(\eps,\delta)$-differential privacy if for all neighboring datasets $z,z'$, i.e. differing in at most one row, it holds that for all events $S$,
    $\Pr[\mathcal{M}(z)\in S]\leq e^\eps\Pr[\mathcal{M}(z')\in S]+\delta$.
\end{definition}

To illustrate the practical privacy effects of such a theoretical guarantee, we present a simple utility bound for any adversary under a differentially private training algorithm.

\begin{lemma}
    For fixed FPR $\alpha$, if training mechanism $\mathcal{M}$ satisfies $(\varepsilon,\delta)$-DP, then the TPR is bounded above by $e^\varepsilon\alpha+\delta$.
\end{lemma}

\begin{proof}
    Denoting arbitrary $z:=\mathcal{D}_{\mathrm{train}}$ and $z':=\mathcal{D}_{\mathrm{train}}\cup\{x\}$ adjacent, we can rewrite 
    $\text{FPR}=\Pr[\hat b=1|H_0]=\Pr[\mathcal{A}(\mathcal{M}(z))=1],$ 
    $\text{TPR}=\Pr[\hat b=1|H_1]=\Pr[\mathcal{A}(\mathcal{M}(z'))=1],$ 
    giving us that
    $\text{TPR}\leq e^\varepsilon \text{FPR}+\delta.$
\end{proof}

\begin{remark}
    The TPR@1\%FPR metric is bounded above by $0.01e^\varepsilon+\delta$.
\end{remark}

\subsection{Experimental setup}

We train a CelebA checkpoint, separately to our main results, with DP-SGD (via
Opacus~\cite{yousefpour2021opacus}) at three privacy budgets, $\epsilon\in\{1,4,10\}$ and $\delta=10^{-5}$, alongside a non-private (`Vanilla') baseline trained identically but without DP-SGD. All four checkpoints share the same architecture, training recipe, and 500-epoch budget, differing only in the presence and strength of
DP-SGD, with $n{=}1000$ member images. We use an identity-aware member/held-out
split: no identity's photos appear on both
sides, preventing the defense evaluation from being confounded by identity-level correlation that DP-SGD's per-example guarantee does not address.

\subsection{Results and Discussion}

Table~\ref{tab:dpsgd_results_n1000} reports ASR, AUC, and TPR@1\%FPR for all
five methods against all four checkpoints. The Vanilla checkpoint shows substantial, consistent memorization across every method: AUC ranges from 77.7\% (SecMI$_\text{stat}$) to 86.9\% (\DIME), with \DIME achieving the strongest result on every metric (79.8\% ASR, 86.9\% AUC, 22.4\%TPR@1\%FPR), which confirms this checkpoint is a genuine, non-trivial target
before any defense is applied.

Under DP-SGD, every method collapses to near-chance  performance at every tested $\epsilon$, including \DIME. AUC falls into a narrow 48.1-52.9\% band
across all fifteen (method, $\epsilon$) combinations, and TPR@1\%FPR drops from Vanilla's 10.2-22.4\% range to at most 1.8\% under any DP-SGD setting. Notably, we do not observe a consistent monotonic trend across
$\epsilon=1,4,10$ for any method; performance at $\epsilon=1$ is not
reliably lower than at $\epsilon=10$, which is explained with all three privacy budgets already suppressing memorization signal below what our evaluation scale can reliably distinguish, rather than $\epsilon$ producing a graded effect in this regime.
These results indicate that DP-SGD remains an effective, complete defense against membership inference on diffusion models even against an attack
with substantially higher power in the non-private setting: \DIME's 8-15
percentage point AUC advantage over the strongest baseline on the Vanilla setting (Table~\ref{tab:main_results}) does not survive contact with DP-SGD at any tested privacy budget.

\section{Conclusion}
\label{sec:conclusion}

We derived an exact, closed-form characterization of the idealized denoiser; the MSE-optimal noise-prediction function a diffusion model would compute at
any query point, given a finite training set, showing it takes the form of a responsibility-weighted average over training examples, with weights determined by a softmax over distance to the query. This
posterior responsibility structure is the exact mechanism by which a finite training set leaves
a detectable trace in the model's output. When responsibility concentrates on a single training point, the model's implicit reconstruction reveals that point's identity with near-certainty; when it
is divided among several, the reconstruction instead reflects 
an aggregate that need not correspond to any single training example. Framing membership inference through this lens replaces an empirically-motivated
search for informative statistics with a question that has a derivable answer: what does the training set's structure make observable at a
given query.

That structure decomposes,
exactly, into two complementary, independently estimable sources of membership signal: how far a candidate's model-implied reconstruction deviates from the candidate itself, and how densely the training points responsible for that reconstruction are crowded around one another. The first channel is what prior diffusion membership inference work measures; the second has no analog in that literature, materializing from our theoretical analysis of the idealized denoiser. We showed both channels admit query-efficient estimators computable from forward evaluations of the trained network alone, sharing the same perturbed queries and requiring no gradient access, shadow models, or generative sampling. We complete the story by developing \DIME, an efficient and motivated membership error derived attack.

Across five diffusion checkpoints spanning a large range in input dimensionality and containing both unconditional and class-conditional architectures, \DIME outperforms
every baseline attack we compare against on ASR, AUC, and TPR@1\%FPR simultaneously, at matched or substantially lower query cost. We observe
cases where our cheapest setting exceeds existing baselines' own most query expensive setting. We additionally evaluated \DIME against DP-SGD, the standard formal defense against privacy adversaries, and found that our attack, like the tested prior baselines, is reduced to near-chance performance at every tested privacy budget. This is itself an informative result: DP-SGD's guarantee holds even against an attack with substantially higher power in the undefended setting.

\paragraph{Limitations and future work.} Our theoretical characterization
concerns the \emph{idealized} denoiser; a trained neural network approximates this object rather than realizing it exactly, and while our empirical results suggest this approximation is close enough to be exploitable,
formalizing the gap
remains an open problem.
Furthermore, the bias and crowding decomposition has a natural coding-theoretic interpretation: the training set acts as a Euclidean codebook and the diffusion denoiser as a soft decoder.
The bias term measures displacement from the decoded codeword, while the crowding term measures ambiguity among nearby plausible codewords. 
Exploring this connection, and its relation to privacy, is important future work~\citep{Alabi26}.

\cleardoublepage
\bibliographystyle{plainurl}
\bibliography{main}

\cleardoublepage
\appendix
\section*{Ethical Considerations}

We describe a potential privacy adversary that could be deployed by malicious actors, although the attack can also be developed benevolently for privacy auditing. We believe the immediate risk of our own work is outweighed by the benefits from increasing public understanding of diffusion model behavior and privacy vulnerabilities. We also present a feasible and effective defense against our attack to counteract potential harms incurred.

\section*{Open Science}

All code and instructions for recreating all results and figures can be found at the anonymous repository linked: \url{https://anonymous.4open.science/r/dime-mia-E1C5/}.

\section{Additional Theoretical Proofs and Details}
\label{sec:appendix-theory}

\begin{lemma}
    Consider a sequence of functions 
    $$r_i^{(j)}(g)=\frac{s_i^{(j)}(g)}{\sum_k s_k^{(j)}(g)},$$ where $s_i^{(j)}(g)=\exp\left(-\frac{1}{2(\sigma^{(j)})^2}\Vert g-a^{(j)}x_i\Vert^2\right)$, $a^{(j)}=\sqrt{1-\sigma^{(j)2}}$, and $\sigma^{(j)}\rightarrow 0^+$. Then
    $$r_i^{(j)}(g)\longrightarrow\mathbf{1}_{V_i}(g)\quad\text{a.e.}$$
    \label{lem:convergence-denoiser-appendix}
\end{lemma}

\begin{proof}
    We have that
    \begin{align*}
        r_i^{(j)}(g)=\frac{s_i^{(j)}(g)}{\sum_k s_k^{(j)}(g)}=\frac{1}{1+\sum_{k\neq i}\frac{s_k^{(j)}(g)}{s_i^{(j)}(g)}}.
    \end{align*}
    We want to show that
    $$\frac{s_k^{(j)}(g)}{s_i^{(j)}(g)}\longrightarrow\begin{cases}
        0: & d_i(g)< d_k(g) \\
        \infty: & d_i(g) > d_k(g)
    \end{cases}.$$
    Suppose the above identity holds, then considering any $g$ such that $d_i(g)\neq d_k(g)$ for any $i\neq k$, then letting $i^*$ be the unique nearest index, i.e. $d_{i^*}(g):=\min_i d_i(g)$, we have two cases:
    \begin{itemize}
        \item If $i=i^*,$ then every $k\neq i$ has $d_i(g)<d_k(g)$, so
        $$r_i^{(j)}(g)=\frac{1}{1+\sum_{k\neq i}\frac{s_k^{(j)}(g)}{s_i^{(j)}(g)}}\rightarrow\frac{1}{1+0}=1.$$
        \item If $i\neq i^*,$ then $d_{i^*}(g)<d_i(g)$, so (observing all ratios are positive)
        $$r_i^{(j)}(g)=\frac{1}{1+\sum_{k\neq i}\frac{s_k^{(j)}(g)}{s_i^{(j)}(g)}}\leq\frac{1}{1+\frac{s_{i^*}^{(j)}(g)}{s_i^{(j)}(g)}}\rightarrow 0.$$
    \end{itemize}
    So outside of the zero-measure set 
    $$\{g:\exists\, i\neq k\text{ s.t. }d_i(g)=d_k(g)\},$$ it holds that $r_i^{(j)}(g)\rightarrow \mathbf{1}_{V_i}$. It remains to show the aforementioned identity. For notational convenience, let $\sigma=\sigma^{(j)}$, letting $a:=\sqrt{1-\sigma^2}$ and $s_i(g)=\exp(-\frac{1}{2\sigma^2}\Vert g-ax_i\Vert^2)$. We have
    $$\frac{s_k(g)}{s_i(g)}=\exp\left(-\frac{1}{2\sigma^2}\Vert g-ax_k\Vert^2+\frac{1}{2\sigma^2}\Vert g-ax_i\Vert^2\right)$$
    $$=\exp\left(\frac{a}{\sigma^2}\langle g,x_k-x_i\rangle-\frac{a^2}{2\sigma^2}(\Vert x_k\Vert^2-\Vert x_i\Vert^2)\right)$$
    \begin{align*}
        = \exp\Bigl( & \frac{a}{\sigma^2}\Bigl[\langle g,x_k-x_i\rangle-\frac{1}{2}(\Vert x_k\Vert^2-\Vert x_i\Vert^2) \\
        &+\frac{1-a}{2}(\Vert x_k\Vert^2-\Vert x_i\Vert^2)\Bigr]\Bigr)
    \end{align*}
    $$=\exp\left(\frac{a}{2\sigma^2}(d_i(g)^2-d_k(g)^2)+\frac{a(1-a)}{2\sigma^2}(\Vert x_k\Vert^2-\Vert x_i\Vert^2)\right).$$
    We can write
    $$\frac{a}{\sigma^2}=\frac{\sqrt{1-\sigma^2}}{\sigma^2}\rightarrow \infty\quad\text{as }\sigma\rightarrow 0^+.$$
    On the other hand, it holds that
    $$\frac{a(1-a)}{2\sigma^2}=\frac{a(1-a)}{2(1-a^2)}=\frac{a}{2(1+a)}\rightarrow\frac{1}{4}\;\;\text{as }\sigma\rightarrow 0^+.$$
    So within the exponential, we have one term that blows up and one term that becomes irrelevant asymptotically. Subsequently, the convergence behavior depends on the sign of $d_i(g)^2-d_k(g)^2$, i.e. whether $g$ is closer to $x_i$ or $x_k$, and we conclude as desired that
    $$\frac{s_k(g)}{s_i(g)}\longrightarrow\begin{cases}
        0: & d_i(g)< d_k(g) \\
        \infty: & d_i(g) > d_k(g)
    \end{cases}\quad\text{as }\sigma\rightarrow 0^+.$$
\end{proof}

\begin{corollary}[Ideal denoiser as a scaled bias vector]
    For any query point $g$ and timestep $t$, we can write
    \begin{equation}
        \hat\varepsilon^*(g,t)=\frac{1}{\sigma_t}(g-\mathbb{E}[Y]),
    \end{equation}
    where $Y=\sqrt{\overline\alpha_t}I\sim r(g)$ is the (scaled) training point random variable denoted in~\eqref{eq:bias_variance}.
    \label{cor:error-decomposition}
\end{corollary}

\begin{proof}
    By definition, $\mathbb{E}[Y] = \sqrt{\overline\alpha_t}\sum_i r_i(g)x_i$.
    Substituting into Lemma~\ref{lem:optimal-denoiser},
    \[
    \hat\epsilon^*(g,t) = \frac{g}{\sigma_t} - \frac{\sqrt{\overline\alpha_t}}{\sigma_t}\sum_i r_i(g)x_i
    = \frac{g}{\sigma_t} - \frac{\mathbb{E}[Y]}{\sigma_t} = \frac{1}{\sigma_t}\big(g-\mathbb{E}[Y]\big).
    \]
\end{proof}

\begin{lemma}[Error of the squared-bias estimator]
\label{lem:squared-bias-estimator}
Fix a timestep $t$ and a query point $g$, and write
\[
    f(x) := \hat{\epsilon}^{\star}(x,t),
    \qquad
    B := \|f(g)\|_2.
\]
Let $z_1,\ldots,z_q \stackrel{\mathrm{i.i.d.}}{\sim}
\operatorname{Rad}(\{-1,+1\}^d)$, and define
\[
    \hat B
    :=
    \frac{1}{q}\sum_{k=1}^q
    \left\|
        f(g+\gamma\sigma_t z_k)
    \right\|_2.
\]

Suppose that, throughout the perturbation neighborhood
\[
    \mathcal N_\gamma(g)
    :=
    \left\{
        g+s\gamma\sigma_t z:
        s\in[0,1],\;
        z\in\{-1,+1\}^d
    \right\},
\]
there exist constants $L_t,H_t<\infty$ such that
\[
    \|f(g+h)-f(g)\|_2
    \leq L_t\|h\|_2
\]
and
\[
    \|f(g+h)-f(g)-J h\|_2
    \leq \frac{H_t}{2}\|h\|_2^2,
    \qquad
    J:=\nabla_g f(g).
\]
Then
\[
    -B H_t\gamma^2\sigma_t^2 d
    \leq
    \mathbb E[\hat B^2]-B^2
    \leq
    \bigl(BH_t+L_t^2\bigr)
    \gamma^2\sigma_t^2 d.
\]
Consequently,
\[
    \left|
        \mathbb E[\hat B^2]-B^2
    \right|
    \leq
    \bigl(BH_t+L_t^2\bigr)
    \gamma^2\sigma_t^2 d,
\]
and hence, for fixed $t$, $g$, and $d$,
\[
    \mathbb E[\hat B^2]
    =
    B^2+O(\gamma^2).
\]
\end{lemma}

\begin{proof}
Let
\[
    a:=f(g),
    \qquad
    h:=\gamma\sigma_t z,
    \qquad
    \Delta(z):=f(g+h)-f(g).
\]
Because $z$ is Rademacher,
\[
    \mathbb E[h]=0,
    \qquad
    \|h\|_2^2=\gamma^2\sigma_t^2d.
\]
By the assumed second-order expansion,
\[
    \Delta(z)=Jh+R(h),
    \qquad
    \|R(h)\|_2
    \leq
    \frac{H_t}{2}\|h\|_2^2.
\]
Therefore,
\begin{equation}
    \left\|\mathbb E[\Delta(z)]\right\|_2
    =
    \left\|\mathbb E[R(h)]\right\|_2
    \leq
    \frac{H_t}{2}
    \gamma^2\sigma_t^2d.
\label{eq:expected-delta}
\end{equation}
The Lipschitz assumption also gives
\begin{equation}
    \|\Delta(z)\|_2
    \leq
    L_t\gamma\sigma_t\sqrt d.
\label{eq:delta-bound}
\end{equation}

For one perturbation, define
\[
    X(z):=\|a+\Delta(z)\|_2.
\]
By convexity of the square,
\[
    \hat B^2
    =
    \left(\frac1q\sum_{k=1}^q X(z_k)\right)^2
    \leq
    \frac1q\sum_{k=1}^q X(z_k)^2.
\]
Taking expectations and expanding the squared norm,
\begin{align*}
    \mathbb E[\hat B^2]
    &\leq
    \mathbb E[X(z)^2] \\
    &=
    \mathbb E\|a+\Delta(z)\|_2^2 \\
    &=
    \|a\|_2^2
    +
    2\left\langle
        a,\mathbb E[\Delta(z)]
    \right\rangle
    +
    \mathbb E\|\Delta(z)\|_2^2.
\end{align*}
Using Equations~\ref{eq:expected-delta},~\ref{eq:delta-bound}, and $\|a\|_2=B$,
\begin{equation}
    \mathbb E[\hat B^2]
    \leq
    B^2+
    \bigl(BH_t+L_t^2\bigr)
    \gamma^2\sigma_t^2d.
    \label{eq:expected-b-hat-squared-upper}
\end{equation}

For the opposite direction, Jensen's inequality gives
\[
    \mathbb E[\hat B^2]
    \geq
    \bigl(\mathbb E[\hat B]\bigr)^2.
\]
Moreover,
\begin{align*}
    \mathbb E[\hat B]
    &=
    \mathbb E_z\|a+\Delta(z)\|_2 \\
    &\geq
    \left\|
        a+\mathbb E[\Delta(z)]
    \right\|_2.
\end{align*}
It follows that
\begin{align}
    \mathbb E[\hat B^2]
    &\geq
    \left\|
        a+\mathbb E[\Delta(z)]
    \right\|_2^2 \\
    &=
    B^2
    +
    2\left\langle
        a,\mathbb E[\Delta(z)]
    \right\rangle
    +
    \left\|\mathbb E[\Delta(z)]\right\|_2^2 \\
    &\geq
    B^2
    -
    2B\left\|\mathbb E[\Delta(z)]\right\|_2 \\
    &\geq
    B^2-BH_t\gamma^2\sigma_t^2d.
   \label{eq:expected-b-hat-squared-lower}
\end{align}
Combining \eqref{eq:expected-b-hat-squared-upper} and \eqref{eq:expected-b-hat-squared-lower} proves the result.
\end{proof}

\begin{remark}[Why Lemma~\ref{lem:squared-bias-estimator} concerns $\hat B^2$]
The unsquared estimator $\hat B$ does not, in general,
satisfy $\mathbb E[\hat B]=B+O(\gamma^2)$. The norm is
nondifferentiable at points where
$\hat\epsilon^\star(g,t)=0$, and at such points the bias can
be $\Theta(\gamma)$. For example, in one dimension, if
$f(u)=u/\sigma_t$ and $g=0$, then $B=0$ while
$\hat B=\gamma$ for every Rademacher perturbation.
Nevertheless, $\hat B^2=\gamma^2$, which is consistent with
Lemma~\ref{lem:squared-bias-estimator}. This is also the relevant
quantity because the exact reconstruction error decomposition
contains $B^2$, not $B$.
\end{remark}

\begin{lemma}[Error order of crowding estimator $\hat V$]
    Fix a query point $x^*$, timestep $t$, and let $g=\sqrt{\bar\alpha_t}x^*$.
    Let $I\sim r(g)$,     $Y=\sqrt{\bar\alpha_t}x_I\in\mathbb{R}^d$, and
    $V=\mathrm{tr}(\mathrm{Cov}(Y))$. For step size $\gamma>0$ and i.i.d.\ Rademacher vectors $z_1,\dots,z_q\in\{-1,+1\}^d$, define
    \begin{equation}
        \hat V(x^*,t,\gamma;q) := \frac{1}{q}\sum_{k=1}^q \frac{z_k^\top\big(\hat\epsilon^*(g+\gamma\sigma_t z_k,\,t)-\hat\epsilon^*(g,t)\big)}{\gamma\sigma_t}.
    \label{eq:recap_dhat}
    \end{equation}
    Then
    \begin{equation}
        \mathbb{E}[\hat V] = \frac{d}{\sigma_t} - \frac{1}{\sigma_t^3}\,V + O(\gamma),
    \label{eq:main_claim}
    \end{equation}
    where $d$ is the input dimension.
    \label{lem:V-error-bias}
\end{lemma}

\begin{proof}
    Let 
    $$m(g):=\sum_i r_i(g)x_i, \quad l_i(g):=-\|g-\sqrt{\bar\alpha_t}x_i\|^2/(2\sigma_t^2),$$
    so $r_i(g)=\exp(l_i(g))/\sum_j\exp(l_j(g))$. Differentiating,
    $$\nabla_g l_i(g) = -(g-\sqrt{\overline\alpha_t}x_i)/\sigma_t^2,$$ and the standard
    softmax-Jacobian identity gives
    \[
    \nabla_g r_i(g) = r_i(g)\cdot\frac{\sqrt{\overline\alpha_t}\big(x_i-m(g)\big)}{\sigma_t^2}.
    \]
    Using 
    $$\nabla_g m(g) = \sum_i x_i\,\nabla_g r_i(g)^\top$$ 
    and
    $$\sum_i r_i(g)\cdot (x_i-m(g))(x_i-m(g))^\top = \mathrm{Cov}_{i\sim r(g)}(x_i),$$
    it holds that (using $m:=m(g)$ and $r_i:=r_i(g)$ as shorthand)
    \begin{align*}
        \nabla_g m(g)=\frac{\sqrt{\overline\alpha_t}}{\sigma_t^2}\left[\sum_i r_i(x_i-m)(x_i-m)^\top+ m\sum_ir_i (x_i-m)^\top\right],
    \end{align*}
    and 
    $$\sum_i r_i(x_i-m)^\top=\left(\sum_i r_ix_i\right)^\top-m^\top\sum_ir_i=m^\top-m^\top=0,$$
    so it holds that
    \[
    \nabla_g m(g) = \frac{\sqrt{\overline\alpha_t}}{\sigma_t^2}\,\mathrm{Cov}_{i\sim r(g)}(x_i).
    \]
    Since 
    $$C=\mathrm{Cov}(\sqrt{\overline\alpha_t}x_I)=\overline\alpha_t\,\mathrm{Cov}_{i\sim r(g)}(x_i),$$
    then we have
    $$\nabla_g m(g) = C/(\sqrt{\overline\alpha_t}\sigma_t^2).$$ 
    Writing
    $J:=\nabla_g\hat\varepsilon^*(g,t)$, where $\hat\varepsilon^*(g,t)$ is defined as in Lemma~\ref{lem:optimal-denoiser}, we have
    \begin{align}
        & J = \frac{1}{\sigma_t}I_d - \frac{\sqrt{\overline\alpha_t}}{\sigma_t}\cdot\frac{C}{\sqrt{\overline\alpha_t}\,\sigma_t^2}
        = \frac{1}{\sigma_t}I_d - \frac{1}{\sigma_t^3}C,
    \\
    &\mathrm{tr}(J) = \frac{d}{\sigma_t} - \frac{1}{\sigma_t^3}\,\mathrm{tr}(C).
    \label{eq:trace_relation}
    \end{align}

    By the Hutchinson trace
    estimator~\citep{hutchinson1989stochastic}, $\mathbb{E}_z[z^\top Jz]=\mathrm{tr}(J)$
    for any fixed matrix $J$ and Rademacher draws $z$. We do not have access to $Jz$
    directly, only to forward evaluations of $\hat\epsilon^*$; by Taylor's
    theorem, for $h=\gamma\sigma_tz_k$,
    \[
    \frac{z_k^\top\big(\hat\epsilon^*(g+\gamma\sigma_tz_k,t)-\hat\epsilon^*(g,t)\big)}{\gamma\sigma_t}
    = z_k^\top Jz_k + \frac{O(\Vert h\Vert^2)}{\gamma\sigma_t},
    \]
    and for fixed $z_k,\sigma_t$ it holds that
    $$O(\Vert h\Vert^2)=O(\gamma^2\sigma_t^2\Vert z_k\Vert^2)=O(\gamma^2),$$
    so we have that
    \[
    \frac{z_k^\top\big(\hat\epsilon^*(g+\gamma\sigma_tz_k,t)-\hat\epsilon^*(g,t)\big)}{\gamma\sigma_t}
    = z_k^\top Jz_k + O(\gamma).
    \]
    Averaging over $q$ draws and taking expectation,
    \begin{equation}
    \mathbb{E}[\hat V] = \mathrm{tr}(J) + O(\gamma).
    \label{eq:dhat_estimates_trJ}
    \end{equation}
    This $O(\gamma)$ term is specific to our finite-step approximation of the
    exact matrix-vector product assumed by the standard estimator, and is the
    only respect in which $\hat V$ deviates from it.
     
    Finally, substituting Equation~\ref{eq:trace_relation}
    into Equation~\ref{eq:dhat_estimates_trJ} gives Equation~\ref{eq:main_claim}.
\end{proof}
\section{Synthetic Visualizations}
\label{sec:appendix-experiments}

\begin{figure*}[t]
    \centering
    \includegraphics[width=\textwidth]{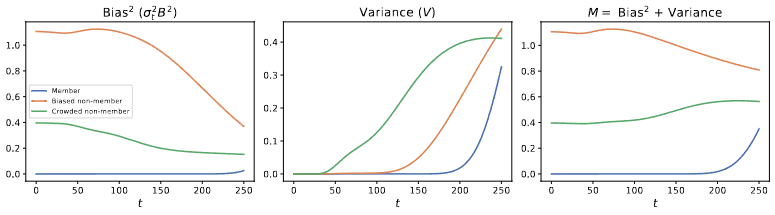}
    \caption{Membership signal of Bias and Variance term from the decomposition in Equation~\ref{eq:M_decomposition} from idealized denoiser over simulated data. The variance term is more effective at capturing membership signal of held-out data in a dense region of member data.}
    \label{fig:bias-variance}
\end{figure*}

\begin{figure*}[t]
    \centering
    \includegraphics[width=\textwidth]{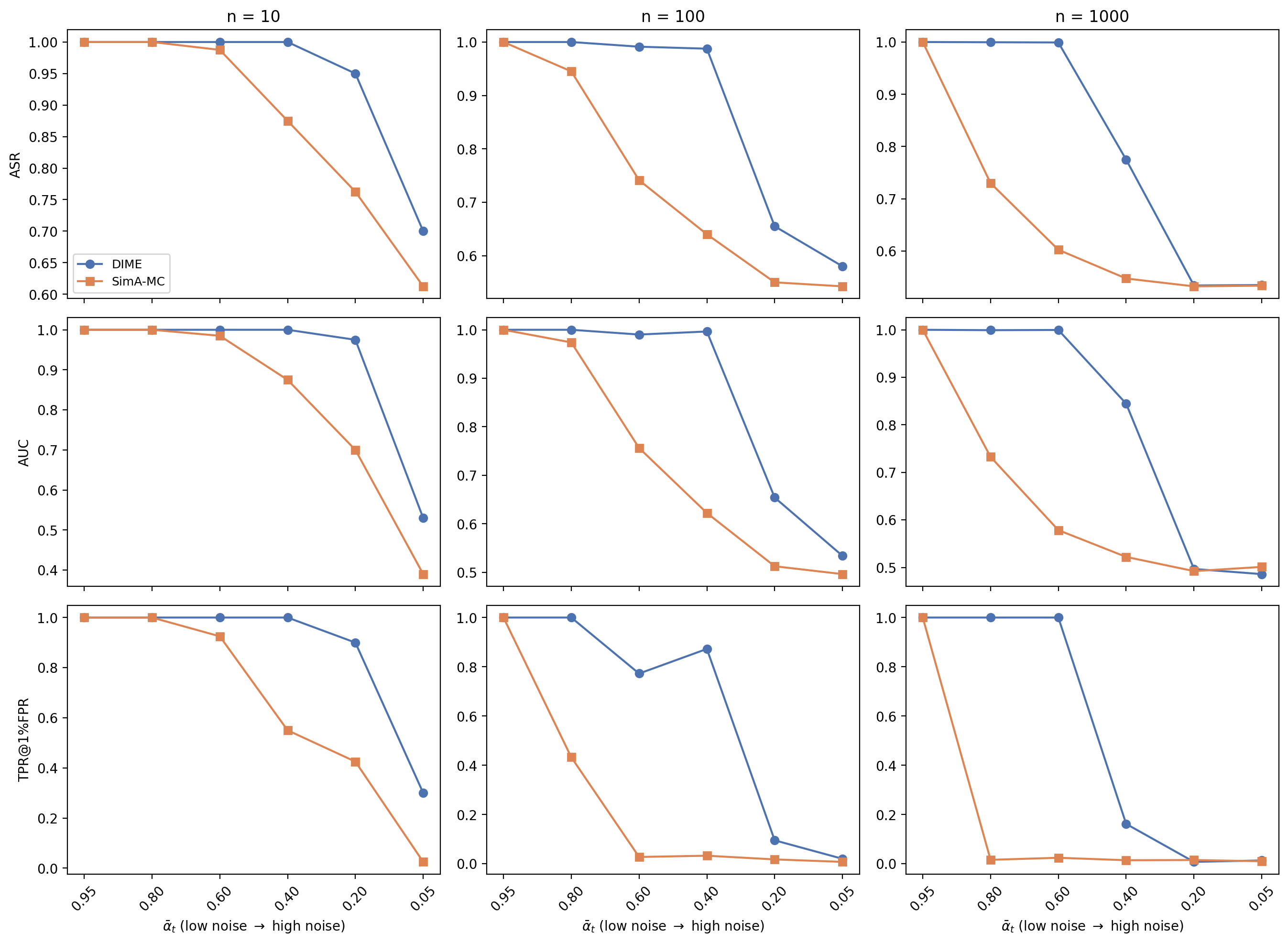}
    \caption{\DIME vs. SimA-MC for the idealized denoiser on simulated data with $q=10$, over various dataset sizes. The additional crowding signal captured in \DIME improves membership signal decay with more diffusion noise.}
    \label{fig:q-fixed-grid}
\end{figure*}

\begin{figure*}[t]
    \centering
    \includegraphics[width=\textwidth]{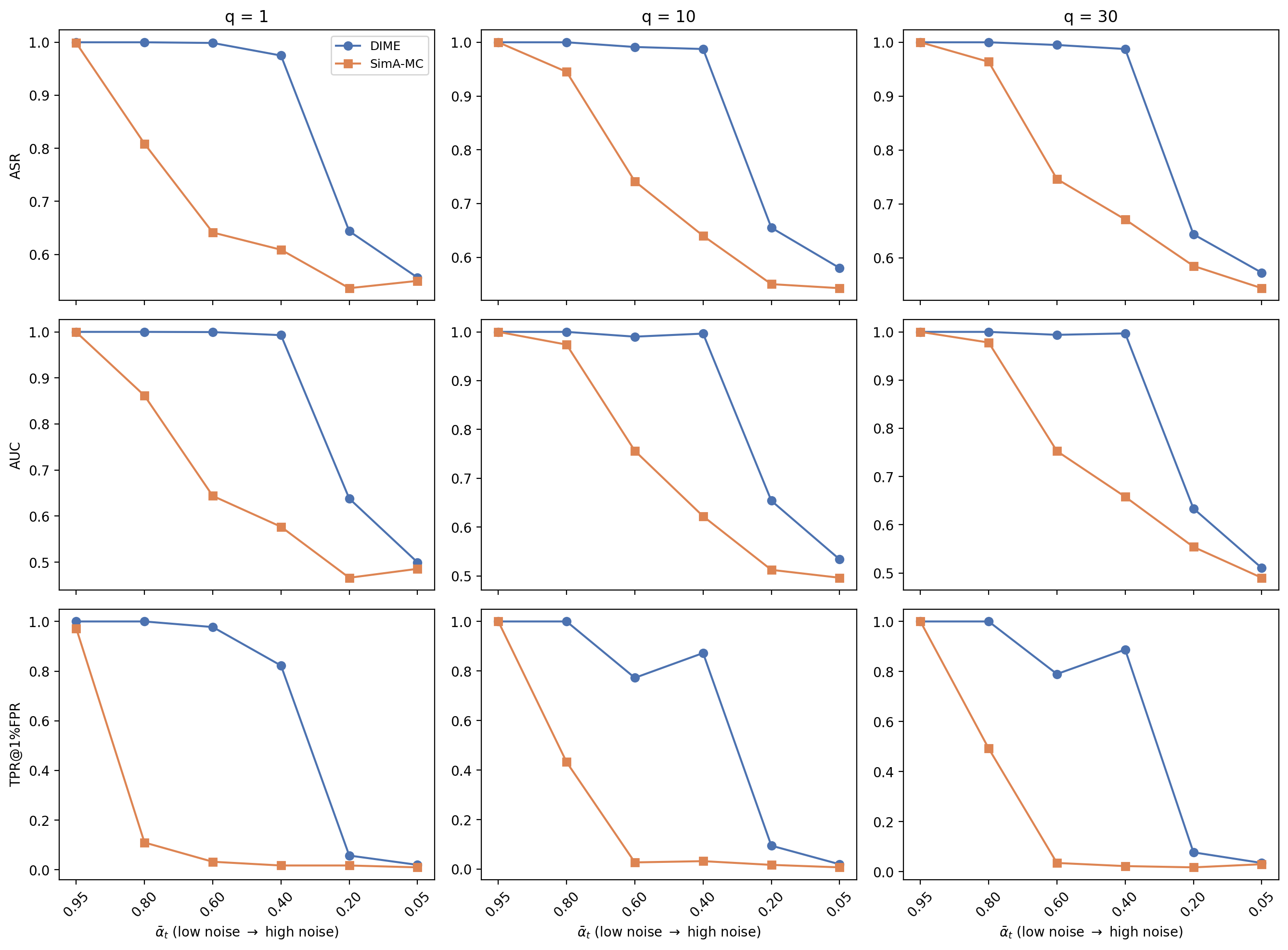}
    \caption{\DIME vs. SimA-MC for the idealized denoiser on simulated data with $n=100$, over various query budgets. The additional crowding signal captured in \DIME improves membership signal decay with more diffusion noise.}
    \label{fig:n-fixed-grid}
\end{figure*}

\begin{figure*}[t]
    \centering
    \includegraphics[width=\textwidth]{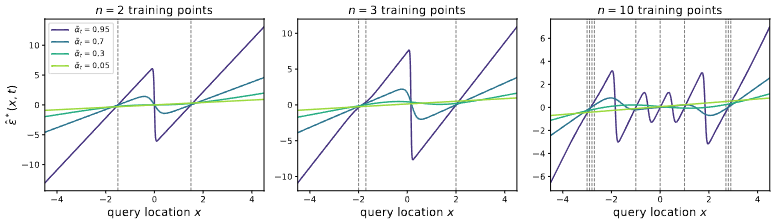}
    \caption{\textbf{The idealized denoiser on synthetic 1D data.} $\hat\epsilon^*(x,t)$, computed exactly from the closed-form optimal denoiser, on synthetic training sets of size $n=2,3,10$ (vertical dashed lines mark training points), at four noise levels $\overline\alpha_t$.}
    \label{fig:synthetic-denoiser}
\end{figure*}

This appendix supplements the main body with experiments conducted in a
fully synthetic, idealized setting: rather than a trained neural network
approximating $\hat\epsilon^*$, we use the exact closed-form idealized denoiser (Lemma~\ref{lem:optimal-denoiser}) directly, computed exactly from a small, explicitly constructed training set. This removes any
confound from network approximation error, isolating whether the theoretical constructions and empirical trends described in the main body hold true of the underlying \emph{ideal} object itself. These experiments
are illustrative and diagnostic rather than a substitute for the real-network results in Section~\ref{sec:experiments}, and are not intended to be read as a continuous narrative with one another.
 
\subsection{Validating the bias-variance decomposition}
 
Figure~\ref{fig:bias-variance} directly validates the exact decomposition
$M = \sigma_t^2B^2 + V$ (Equation~\ref{eq:M_decomposition})
against the three-case construction motivating our approach
(Section~\ref{sec:theory}): a genuinely isolated member, an isolated
non-member whose nearest training point lies elsewhere (\emph{biased}),
and a non-member embedded in a dense cluster of training points whose
responsibility-weighted centroid happens to coincide with the candidate
(\emph{crowded}). The bias$^2$ panel confirms the intended construction:
the isolated member's bias remains near zero across nearly the entire
displayed range, while both non-member constructions show substantial
bias, with the crowded case's bias notably lower than the biased case's
by design. The variance panel shows the key, motivating asymmetry: the
crowded non-member's variance rises far earlier (before $t\approx50$)
than either the member's or the biased non-member's, both of which remain
near zero until much later in the range. This is precisely the signature predicted in Section~\ref{sec:theory}: a candidate whose bias alone would appear small can still be correctly flagged once variance is taken into account, and the earliest, clearest such signal in this construction comes
specifically from the crowded case. We note the displayed range is
deliberately restricted to where this asymmetry is visible; at
sufficiently large $t$, even the isolated member's estimation error rises
as noise begins to overwhelm any fixed geometric separation, consistent
with Lemma~\ref{lem:convergence-denoiser-appendix}'s asymptotic result being specifically
a small-$t$ statement.
 
\subsection{DIME versus SimA-MC across dataset size, at fixed query budget}
 
Figure~\ref{fig:q-fixed-grid} compares DIME against SimA-MC in the
idealized setting across three synthetic dataset sizes ($n=10,100,1000$),
at a fixed query budget, with query-averaged attack effectiveness plotted
against the noise level $\overline\alpha_t$ used at query time. DIME
outperforms SimA-MC across every metric (ASR, AUC, TPR@1\%FPR), every
dataset size, and nearly every noise level, with the gap most pronounced
at higher noise levels, where SimA-MC's performance collapses toward
chance while DIME retains substantially more signal. This advantage
persists as $n$ grows across two orders of magnitude, suggesting DIME's
improvement over a level-only statistic is not an artifact of the small,
low-dimensional training sets used in this synthetic construction.
 
\subsection{DIME versus SimA-MC across query budget, at fixed dataset size}
 
Figure~\ref{fig:n-fixed-grid} complements the above by fixing dataset size
and instead varying the query budget ($q=1,10,30$). The same qualitative
pattern holds: \DIME's advantage over SimA-MC is consistent across all
three query budgets, indicating the gap is not primarily an artifact of
how many queries either method is given, but reflects a persistent
difference in the two statistics themselves. We note that both methods'
bias components share the same $q$ perturbed queries in this
construction, and \DIME's crowding term is computed via the exact closed-form Jacobian trace rather than the finite-sample Hutchinson estimator used in our real-network experiments (Section~\ref{sec:theory});
this synthetic comparison therefore reflects the idealized, zero-estimation-noise
crowding signal, not the practical $q+1$-query estimator's behavior.
 
\subsection{Illustrating the idealized denoiser in one dimension}
 
Figure~\ref{fig:synthetic-denoiser} visualizes $\hat\varepsilon^*(x,t)$
directly on simple, one-dimensional synthetic training sets of size $n=2,3,10$, at several noise levels, with vertical dashed lines marking
training point locations. Two properties predicted by
Corollary~\ref{cor:error-decomposition} are directly visible: first, at low noise (dark curves), $\hat\varepsilon^*$ crosses zero \emph{exactly} at each
training point and grows as the query moves away from it, consistent with $\hat\epsilon^*$ being proportional to the bias vector $(x-m(x))$; second,
as noise increases (lighter curves), this structure is progressively
smoothed away, consistent with responsibility spreading across the
training set rather than concentrating on individual points. The $n=10$
panel additionally illustrates richer, multi-modal behavior near the two
deliberately-constructed clusters (points near $x=-3$ and $x=3$),
foreshadowing the crowding phenomenon central to
Section~\ref{sec:theory}.

\end{document}